\documentclass[11pt]{amsart}

\usepackage[margin=1in]{geometry}
\usepackage{amsmath,amssymb,amsthm,mathtools}
\usepackage{graphicx}
\usepackage{xcolor}
\usepackage{hyperref}
\usepackage[nameinlink,capitalize]{cleveref}

\usepackage{lipsum}
\usepackage{amsfonts}
\usepackage{graphicx}
\usepackage{epstopdf}
\ifpdf
  \DeclareGraphicsExtensions{.eps,.pdf,.png,.jpg}
\else
  \DeclareGraphicsExtensions{.eps}
\fi

\makeatletter
\renewcommand{\paragraph}{%
  \@startsection{paragraph}{4}%
  {\z@}{\z@}{-\fontdimen2\font}%
  {\normalfont\bfseries}}
\makeatother

\newtheorem{theorem}{Theorem}[section]

\theoremstyle{plain}
\newtheorem{lemma}[theorem]{Lemma}
\newtheorem{corollary}[theorem]{Corollary}
\newtheorem{proposition}[theorem]{Proposition}

\theoremstyle{definition}
\newtheorem{definition}[theorem]{Definition}
\newtheorem{assumption}[theorem]{Assumption}

\theoremstyle{remark}
\newtheorem{remark}[theorem]{Remark}

\crefname{claim}{claim}{claims}
\Crefname{claim}{Claim}{Claims}

\crefname{remark}{remark}{remarks}
\Crefname{remark}{Remark}{Remarks}

\crefname{hypothesis}{hypothesis}{hypotheses}
\Crefname{hypothesis}{Hypothesis}{Hypotheses}

\title[Recovery Theory for Projected Power Iterations]
{Recovery Theory for Projected Power Iterations in Permutation Synchronization}

\author{Vahan Huroyan}

\author{Gilad Lerman}

\thanks{Vahan Huroyan: Department of Mathematics and Statistics,
Saint Louis University, Saint Louis, Missouri, USA.
E-mail: \texttt{vahan.huroyan@slu.edu}.}

\thanks{Gilad Lerman: School of Mathematics,
University of Minnesota, Minneapolis, Minnesota, USA.
E-mail: \texttt{lerman@umn.edu}.}

\thanks{This work was supported by NSF award DMS-2427955.}

\usepackage{amsopn}

\makeatletter
\newcommand*{\addFileDependency}[1]{
  \typeout{(#1)}
  \@addtofilelist{#1}
  \IfFileExists{#1}{}{\typeout{No file #1.}}
}
\makeatother

\usepackage{amsmath,amssymb}
\usepackage{algorithm}
\usepackage{algpseudocode}
\usepackage{cleveref}
\usepackage[nocompress]{cite}

\newcommand{\vzero}{\boldsymbol{0}}
\newcommand{\vone}{\boldsymbol{1}}

\newcommand{\mA}{\boldsymbol{A}}
\newcommand{\mB}{\boldsymbol{B}}

\newcommand{\mD}{\boldsymbol{D}}
\newcommand{\mE}{\boldsymbol{E}}
\newcommand{\mG}{\boldsymbol{G}}
\newcommand{\mH}{\boldsymbol{H}}
\newcommand{\mI}{\boldsymbol{I}}
\newcommand{\mJ}{\boldsymbol{J}}
\newcommand{\mK}{\boldsymbol{K}}
\newcommand{\mL}{\boldsymbol{L}}
\newcommand{\mM}{\boldsymbol{M}}
\newcommand{\mN}{\boldsymbol{N}}
\newcommand{\mP}{\boldsymbol{P}}
\newcommand{\mQ}{\boldsymbol{Q}}
\newcommand{\mO}{\boldsymbol{O}}
\newcommand{\mR}{\boldsymbol{R}}
\newcommand{\mS}{\boldsymbol{S}}
\newcommand{\mU}{\boldsymbol{U}}
\newcommand{\mV}{\boldsymbol{V}}

\newcommand{\mY}{\boldsymbol{Y}}

\newcommand{\mW}{\boldsymbol{W}}

\newcommand{\mlamb}{\boldsymbol{\Lambda}}

\newcommand{\bvD}{\boldsymbol{\bar{\mD}}}
\newcommand{\bvE}{\boldsymbol{\bar{\mE}}}
\newcommand{\bvP}{\boldsymbol{\bar{\mP}}}
\newcommand{\bvQ}{\boldsymbol{\bar{\mQ}}}
\newcommand{\bvR}{\boldsymbol{\bar{\mR}}}

\newcommand{\EE}{\mathbb{E}}

\newcommand{\tr}{\mbox{tr}}
\newcommand{\Var}{\text{Var}}

\newcommand{\argmax}{\operatornamewithlimits{argmax}}
\newcommand{\argmin}{\operatornamewithlimits{argmin}}

\newcommand{\Proj}{\mathcal{P}_{\Delta}}

\newcommand{\Prob}{\mathrm{Prob}}

\newcommand{\Perm}{\text{Perm}}

\newcommand{\ppm}{\text{ppm}}

\ifpdf
\hypersetup{
  pdftitle={Recovery Theory for Projected Power Iterations in Permutation Synchronization},
  pdfauthor={V. Huroyan, G. Lerman}
}
\fi
\hypersetup{hidelinks}



\begin{document}

\maketitle


\begin{abstract}
We study the projected power method for permutation synchronization: recovering $n$ unknown permutations of $m$ objects from pairwise relative permutations, observed independently with probability $p$ and independently corrupted by uniform permutations with probability $1-\pi_0$.
Our main theorem establishes uniform one-step contraction for all estimates with optimally aligned block-error fraction at most $1/2-\epsilon$, for fixed $\epsilon>0$, including measurement-dependent estimates.
When $np\ge C_0\log n$ and $m=o(np\pi_0^2)$, one update contracts the error by a factor $O(m/(np\pi_0^2))$, up to a floor of $O(e^{-c\min\{m\,np\pi_0^2,\,np\pi_0\}}+\log n/n)$.
Thus, every initialization in this majority basin achieves $o(1)$ error after one update and maintains it uniformly over subsequent iterations.
A reference-block spectral initializer enters the basin with probability $1-o(1)$, yielding end-to-end almost-exact recovery.
If additionally $\min\{m\,np\pi_0^2,\,np\pi_0\}\ge C\log n$ for sufficiently large $C$, the method achieves exact recovery in finitely many iterations.
We also establish prescribed-block recovery under weaker conditions for measurement-independent estimates and extend the guarantees to independent permutation-valued corruptions with uniform mean and to partial permutations with common visible support.
\end{abstract}



\section{Introduction}
\label{sec:mgm_intro}

Suppose that \(n\) observations contain the same \(m\) latent objects, in
possibly different orders. The unknown correspondence associated with
observation \(i\) is represented by an absolute permutation
\(\mP_i\in\Perm(m)\). Permutation synchronization seeks to recover
\(\mP_1,\ldots,\mP_n\), up to a common right permutation, from noisy (or corrupted)
observations of the relative permutations \(\mP_i\mP_j^\top\). It provides a
consistency mechanism for multi-way matching: pairwise correspondences
computed independently may disagree around cycles, whereas a family of
absolute permutations induces globally consistent pairwise correspondences.
Applications include shape
matching~\cite{huang2012optimization,huang2013consistent}, multiview
reconstruction~\cite{goesele07multi}, and structure from
motion~\cite{agarwal2011building}.

This paper focuses on a possibly sparse uniform corruption model. Each unordered pair
\(\{i,j\}\) is observed independently with probability \(p\). Conditional on
being observed, its measurement is drawn from the uncorrupted component, and
hence equals \(\mP_i\mP_j^\top\), with probability \(\pi_0\); otherwise it is
an independent uniformly random permutation. Thus, \(np\) is the
expected-degree scale, \(np\pi_0\) is the scale of the expected number of
uncorrupted incident measurements, and \(np\pi_0^2\) is the effective squared
signal-to-noise scale appearing in our guarantees. 
The precise model is given in
Definition~\ref{def:uniform_corruption_model}.

\paragraph{Previous works}
Methods for permutation synchronization include spectral
relaxations~\cite{pachauri2013solving,shen2016normalized}, semidefinite
relaxations~\cite{huang2013consistent}, and direct coordinate
methods~\cite{tang2017initialization}. Spectral and semidefinite
relaxations for group synchronization go back to the angular synchronization
work of Singer~\cite{singer2011angular}. Projected and generalized power
methods have also been studied extensively for phase synchronization,
including convergence and near-optimal statistical
guarantees~\cite{boumal2016nonconvex,zhong2018near}, building on the
tightness analysis of the semidefinite relaxation
in~\cite{bandeira2017tightness}. 
Higher-order cycle information has been used to suppress corrupted
measurements~\cite{lerman2022robust,shi2020robust}. 
The projected power method (PPM) 
updates consist of matrix multiplication followed by separate blockwise assignment problems. Empirical studies
report competitive accuracy and scalability for PPM and related iterative methods in joint alignment and permutation synchronization~\cite{chen2016projected,huroyan2018mathematical,
shi2020robust}.

Chen and Cand\`es~\cite{chen2016projected} analyze projected-power
iterations for cyclic joint alignment and report initial experiments for
permutation synchronization; our notation \(\pi_0\) for the
non-corruption probability follows that paper. Information-theoretic
exact-recovery conditions for pairwise measurements over finite groups,
including permutations, on Erd\H{o}s--R\'enyi and other measurement graphs are given by Chen, Suh, and Goldsmith~\cite{chen2016information}. A preliminary treatment of PPM for permutation synchronization appears in~\cite{huroyan2018mathematical}, but without the basin-uniform trajectory theory developed here.

The closest prior trajectory results are those of Liu, Yue, and
So~\cite{liu2023unified} and Gao and Zhang~\cite{gao2022iterative}.
Liu, Yue, and So establish a deterministic theorem for the generalized
power method that covers the permutation group and yields an
aligned-Frobenius recursion for the initialized trajectory. Their explicit
random-model specialization uses an Erd\H{o}s--R\'enyi measurement graph
with additive i.i.d.\ sub-Gaussian matrix noise. In the permutation
specialization, their initialization condition requires aligned Frobenius
error at most \(\sqrt{n}/8\). Since every incorrectly aligned permutation
block contributes at least \(4\) to the squared Frobenius error, this
condition implies aligned block error at most \(1/256\), and their recursion
uses the fixed geometric factor \(1/\sqrt{2}\). Under an additive Gaussian
observation model, Gao and Zhang prove that the Hamming error of an
iterative permutation estimator contracts geometrically to the minimax
statistical floor.

Under uniform corruption, Lerman and Shi~\cite{lerman2022robust} establish
exact-recovery guarantees for cycle-edge message passing, while
Ling~\cite{ling2022near} proves near-optimal blockwise guarantees and exact
recovery for spectral rounding. Under additive Gaussian noise, Nguyen and
Zhang~\cite{nguyen2025novel} obtain a sharp minimax-optimal error bound for
an aggregated-anchor spectral method. 
We analyze the complete data-dependent block-error trajectory of ordinary unweighted PPM under permutation-valued corruption on possibly sparse observation graphs, assuming \(np\ge C_0\log n\) and \(m=o(np\pi_0^2)\). This includes the intermediate regime \(\min\{m\,np\pi_0^2,\,np\pi_0\}=o(\log n)\), in which our bounds give almost-exact but not exact recovery.

\paragraph{Main results}
The central result is a contraction bound that holds on one event for every estimate with a fixed majority of correct blocks. 
Write \(\delta_k\) for the optimally aligned fraction of incorrect blocks after
\(k\) PPM updates, and write \(r_n:=m/(np\pi_0^2)\) for the basic error scale.
Under \(np\ge C_0\log n\) and \(r_n\to0\),
Theorem~\ref{th:iter_contraction_random_observation} gives
\[
\delta_{k+1}\le\gamma_n\delta_k+\xi_n,
\qquad
\gamma_n=O(r_n),
\qquad
\xi_n=O\!
\left(e^{-c\min\{m\,np\pi_0^2,\,np\pi_0\}}
+\frac{\log n}{n}\right),
\]
whenever \(\delta_k\le1/2-\epsilon\), for fixed
\(\epsilon\in(0,1/2)\). Here \(\gamma_n\) is the contraction factor and \(\xi_n\) is the additive error bound; their explicit definitions are given
in \eqref{eq:contraction_parameters}. The event has probability
\(1-O(n^{-\beta})\) for every fixed \(\beta>0\), and does not depend on
the estimate. This permits reuse of the same measurements at every update.
Since \(\gamma_n,\xi_n\to0\), one update gives vanishing block error from any estimate in the basin, and the basin is invariant thereafter. In particular, \(\sup_{k\ge1}\delta_k=o(1)\) on this event (Theorem~\ref{th:uniform_in_iteration_almost_exact}). We use
\emph{almost-exact recovery} for a vanishing fraction of incorrect blocks,
following the terminology in community detection~\cite{abbe2018community}. 

The reference-block spectral initializer satisfies
\(\EE[\delta_0]=O(r_n)\) and therefore enters the basin with probability
tending to one (Theorem~\ref{th:spectral_initialization}). Its first PPM
update gives the bound
\(\sup_{k\ge1}\delta_k=O_{\mathbb P}(r_n^2+\xi_n)\).
More generally, every fixed \(k\ge0\) satisfies
\(\delta_k=O_{\mathbb P}(r_n^{k+1}+\xi_n)\) (Corollary~\ref{cor:end_to_end_spectral_ppm}). If the stronger
condition
\(c\min\{m\,np\pi_0^2,\,np\pi_0\}-\log n\to\infty\) holds, the exceptional blocks disappear and PPM reaches exact recovery  in finitely many iterations.

For an estimate independent of the measurements, 
$
\log m=o(\min\{m\,np\pi_0^2,\,np\pi_0\})    
$
suffices to recover each prescribed block in one update (Theorem~\ref{th:main_th}). This result explains the local projection margin, but cannot be iterated directly: subsequent estimates depend on the measurements. The uniform contraction bound supplies the missing control. Both arguments extend to independent, not necessarily identically distributed, permutation-valued corruptions with mean \(m^{-1}\vone\vone^\top\). The fixed-reference initializer additionally requires block exchangeability, as stated in Remark~\ref{rem:beyond_uniform_law}.

\paragraph{Signal scales and scope}
The observed- and uncorrupted-degree scales are \(np\) and \(np\pi_0\). The uncorrupted score gap is of order \(np\pi_0\), while the global centered matrices used in the trajectory proof have operator norm of order \(\sqrt{np}\). 
Their comparison gives the squared signal scale \(np\pi_0^2\). 
For a prescribed block, the variance of each corrupted score entry is at most \(np(1-\pi_0)/m\); this yields the finer exponent \(\min\{m\,np\pi_0^2,\,np\pi_0\}\). The trajectory proof must also control arbitrary data-dependent error patterns, and its Frobenius bound introduces the factor \(m\) in \(r_n\).

These assumptions cover sparse observations and decreasing signal. For fixed \(m\) and \(\pi_0>0\), a degree of order \(\log n\) gives end-to-end almost-exact recovery; a sufficiently large constant in that degree also gives exact recovery. If \(p\) and \(\pi_0\) are bounded
below, the trajectory condition permits any sublinear growth \(m=o(n)\). For complete observations and fixed \(m\), almost-exact recovery holds
for \(\pi_0\gg n^{-1/2}\), whereas our exact-recovery guarantee requires \(\pi_0\) to be a sufficiently large constant multiple of
\(\sqrt{\log n/n}\). 
The logarithmic observation-degree condition gives global spectral control; the all-block signal condition separately rules out an error at every block.

The theory also covers the intermediate regime \(m=o(np\pi_0^2)\) and \(\min\{m\,np\pi_0^2,\,np\pi_0\}=o(\log n)\), where the block-error
fraction vanishes but our bounds do not give exact recovery. 
The factor \(m\) in the initializer and contraction bounds remains a limitation when the latent universe is large. Whether it can be replaced by \(\log m\) for a trajectory using the same measurements is open.

Appendix~\ref{sec:partial_permutations} extends the recovery guarantees to partial
permutations with a common visible support through an exact reduction
to full-permutation synchronization. For varying supports, it provides
co-visibility bounds but does not establish a general end-to-end
recovery guarantee.

\paragraph{Structure of this paper}
The rest of this paper is organized as follows. \Cref{sec:mgm_math_setting} formulates permutation synchronization under the
uniform corruption model, and \Cref{sec:mgm_ppm} describes the projected power
iteration and its spectral initialization.
\Cref{sec:mgm_theory} establishes prescribed-block one-step recovery and
uniform basin contraction, derives uniform-in-iteration almost-exact
recovery and finite-step exact recovery, and proves the spectral-initializer
result needed for an end-to-end guarantee. 
Appendix~\ref{sec:partial_permutations} presents the extension to partial permutations.

\section{Mathematical Setting for Permutation Synchronization}
\label{sec:mgm_math_setting}
Let \(\Perm(m)\) denote the set
of \(m\times m\) permutation matrices and let  \(\mP_1,\ldots,\mP_n\in\Perm(m)\). For every \(i\ne j\), we observe a
noisy measurement of the relative permutation
\(\mP_i\mP_j^{-1}=\mP_i\mP_j^\top\). These measurements are stored as the
\(m\times m\) off-diagonal blocks of a matrix
\(\mL\in\mathbb R^{nm\times nm}\).

For exposition, we first describe the complete-observation specialization
of the uniform corruption model. Independently for every \(1\le i<j\le n\),
let
\(
C_{ij}\sim\operatorname{Bernoulli}(\pi_0),
\)
\(
\mU_{ij}\sim\operatorname{Unif}(\Perm(m)),
\)
and define
\begin{equation}
\label{eq:rand_corr}
\mL_{ij}
=
C_{ij}\mP_i\mP_j^\top
+
(1-C_{ij})\mU_{ij},
\text{ and }
\mL_{ji}:=\mL_{ij}^\top
\text{ for }
1\le i<j\le n, 
\text{ and }
\mL_{ii}:=\mI_m.
\end{equation}
Here \(\pi_0\) is the probability of selecting the uncorrupted component,
and \(1-\pi_0\) is the corruption probability. Because a uniformly random
permutation can coincide with the ground-truth relative permutation, the
probability that a measurement equals the ground truth is
\(
\pi_0+{(1-\pi_0)}/{m!}.
\)
The theoretical analysis uses the more general Erd\H{o}s--R\'enyi
observation model of Definition~\ref{def:uniform_corruption_model}; the
complete-observation model above corresponds to \(p=1\).

In the complete-observation model above, if \(\pi_0=1\), then all measurements are uncorrupted and \(\mL=\bvP\bvP^\top\), where \(\bvP=[\mP_1^\top,\ldots,\mP_n^\top]^\top\).
Permutation synchronization asks to determine the permutation matrices
\(\mP_1,\dots,\mP_n\), up to a common right permutation, from \(\mL\).
Let \(\Perm^n(m)\) denote the set of block vectors
\(\bvQ=[\mQ_1^\top,\ldots,\mQ_n^\top]^\top\in\mathbb R^{nm\times m}\)
such that \(\mQ_i\in\Perm(m)\) for every \(i\in[n]\). We consider the
optimization problem
\begin{equation}
\label{eq:perm_sync}
\min_{\bvQ\in\Perm^n(m)}
\|\mL-\bvQ\bvQ^\top\|_F^2.
\end{equation}
Since \(\|\bvQ\bvQ^\top\|_F^2\) is constant over
\(\Perm^n(m)\), \eqref{eq:perm_sync} is equivalent to
\(\max_{\bvQ\in\Perm^n(m)}\tr(\bvQ^\top\mL\bvQ)\).

A standard relaxation replaces the block-permutation constraint by
\(
\bvQ^\top\bvQ=n\mI_m.
\)
Its maximizers are obtained from the top \(m\) eigenvectors of \(\mL\).
The spectral method of Pachauri et al.~\cite{pachauri2013solving} solves
this relaxation and then projects its \(m\times m\) blocks onto
\(\Perm(m)\).
PPM repeatedly multiplies the current estimate by \(\mL\) and projects each resulting block onto \(\Perm(m)\). Every iterate remains feasible for~\eqref{eq:perm_sync}.

\section{Projected Power Method}
\label{sec:mgm_ppm}

In this section we review the Projected Power Method for permutation synchronization. 
We use the zero-diagonal convention of Definition~\ref{def:uniform_corruption_model}; thus, if the measurements are stored with \(\mL_{ii}=\mI_m\) as in~\eqref{eq:rand_corr}, we replace \(\mL\) by \(\mL-\mI_{nm}\). Assume that we are given this preprocessed matrix and an initial estimate
\(\bvQ^{(0)}=[(\mQ_1^{(0)})^\top,\ldots,
(\mQ_n^{(0)})^\top]^\top\), where
\(\mQ_1^{(0)},\ldots,\mQ_n^{(0)}\in\Perm(m)\).
At iteration \(t\ge1\), PPM first computes
\(\widetilde{\bvQ}^{(t)}=\mL\bvQ^{(t-1)}\) and then projects each
\(m\times m\) block of \(\widetilde{\bvQ}^{(t)}\) onto \(\Perm(m)\).

For \(\mA\in\mathbb R^{m\times m}\), define
\(
\Proj(\mA)
\in
\argmax_{\mQ\in\Perm(m)}
\tr(\mQ^\top\mA).
\)
This is the linear assignment problem with weight matrix \(\mA\).
When the maximizer is not unique, we use an arbitrary fixed deterministic tie-breaking rule. Each projection can be computed with the Hungarian algorithm, also known as the Kuhn--Munkres or Munkres assignment algorithm~\cite{kuhn1955hungarian}.

For spectral initialization, let
\(
\mlamb\in\mathbb R^{nm\times m}
\)
have as its columns an orthonormal basis for the leading \(m\)-dimensional
eigenspace of \(\mL\), and let
\(\mlamb_i\in\mathbb R^{m\times m}\) denote its \(i\)-th block. We initialize
the algorithm by setting
\(
\mQ_i^{(0)}
=
\Proj\bigl(n\mlamb_i\mlamb_1^\top\bigr),
\)
\(
i=1,\dots,n.
\)
The product \(\mlamb_i\mlamb_1^\top\) is invariant under replacing
\(\mlamb\) by \(\mlamb \mO\) for any orthogonal matrix \(\mO\), and therefore
does not depend on the choice of basis for the leading eigenspace. The factor
\(n\) does not change the solution of the assignment problem, but places the
population blocks on the scale of permutation matrices.

In \Cref{sec:spectral_initialization}, we prove that this initializer has
vanishing aligned block error under \(np\ge C_0\log n\) and
\(m=o(np\pi_0^2)\). In particular, under the trajectory assumptions it enters
the positive-overlap regime required by
Theorem~\ref{th:iter_contraction_random_observation}. The complete PPM algorithm is summarized in
\Cref{algo:ppm_mgm}.

\begin{algorithm} 
\caption{Projected Power Method for Permutation Synchronization} 
\label{algo:ppm_mgm} 
\begin{algorithmic} 
	\State \textbf{Input}: A symmetric block matrix \(\mL\in\mathbb R^{nm\times nm}\) with \(m\times m\) blocks and zero diagonal blocks, and a stopping parameter \(T_{\ppm}\). 
    \State \textbf{Find} \(\mlamb\in\mathbb R^{nm\times m}\) whose columns form an orthonormal basis for the leading \(m\)-dimensional eigenspace of \(\mL\).
    \State \textbf{Set}: \(\mQ_i^{(0)}=\Proj(n\mlamb_i\mlamb_1^\top)\) for \(1\le i\le n\), and let  \(\bvQ^{(0)} = [(\mQ_1^{(0)})^\top,\ldots,(\mQ_n^{(0)})^\top]^\top\).
\For{\(1\le t\le T_{\ppm}\)}
    \State \textbf{Set}: \(\mQ_i^{(t)}
    =\Proj((\mL\bvQ^{(t-1)})_i)\) for all \(1\le i\le n\), and let
    \(\bvQ^{(t)}
    =[(\mQ_1^{(t)})^\top,\ldots,(\mQ_n^{(t)})^\top]^\top\).
\EndFor    
\State \textbf{Output}: \(\bvQ^{(T_{\ppm})}\).
\end{algorithmic}
\end{algorithm}

\section{Theoretical Analysis}
\label{sec:mgm_theory}

We analyze PPM under the uniform corruption model in the order
outlined in the introduction: one-step recovery for measurement-independent
estimates (\Cref{sec:one_step_recovery}), basin-uniform contraction for
data-dependent estimates (\Cref{sec:basin_uniform_contraction}), its
consequences for the trajectory (\Cref{sec:contraction_across_iterations}),
and the spectral initializer (\Cref{sec:spectral_initialization}).

\subsection{Definitions and Assumptions}
\label{sec:def_assump}
We use the following definitions and assumptions:

\begin{definition}[uniform corruption model]
\label{def:uniform_corruption_model}
Let \(m\ge2\), and let
\(
\mP_1,\dots,\mP_n\in\Perm(m)
\)
denote the ground-truth permutation matrices. Let \(p\in(0,1]\) be an
edge-observation probability, and let \(\pi_0\in(0,1]\) be the probability
of selecting the uncorrupted component conditional on observing an edge.
Independently for every \(1\le i<j\le n\), let
\(
A_{ij}\sim\operatorname{Bernoulli}(p),
\)
\(
C_{ij}\sim\operatorname{Bernoulli}(\pi_0),
\)
\(
\mU_{ij}\sim\operatorname{Unif}(\Perm(m)),
\)
and assume that all these random variables are mutually independent. The
observation indicators \(\{A_{ij}\}_{i<j}\) generate an Erd\H{o}s--R\'enyi
graph \(G(n,p)\).
Define the block measurement matrix
\(\mL\in\mathbb R^{nm\times nm}\) by
\(
\mL_{ij}
=
A_{ij}
(
C_{ij}\mP_i\mP_j^\top
+
(1-C_{ij})\mU_{ij}
),
\)
\(
1\le i<j\le n, 
\)
\(
\mL_{ji}:=\mL_{ij}^\top,
\)
and
\(
\mL_{ii}:=\boldsymbol{0}_{m\times m}.
\)

This is the specialization to the permutation
group of the uniform corruption model for compact-group synchronization
in~\cite{lerman2022robust}, which in turn goes back to the random
corruption models for angular and rotation synchronization
in~\cite{singer2011angular,wang2013exact}. In the notation
\(\operatorname{UCM}(n,p,q_{\mathrm{corr}})\) used there,
\(
q_{\mathrm{corr}}=1-\pi_0.
\)
If the input matrix has identity diagonal blocks, we replace it by
\(\mL-\mI_{nm}\) before applying the projected power iteration.
\end{definition}

\begin{remark}[Extension beyond the uniform law]
\label{rem:beyond_uniform_law}
The one-step and uniform-contraction arguments use the uniform law only
through the following properties: the matrices
\(\{\mU_{ij}:i<j\}\) are independent, permutation-valued, independent of
the observation and non-corruption indicators, and satisfy
\(\EE[\mU_{ij}]=m^{-1}\vone\vone^\top\) for every \(i<j\). Under the remaining assumptions of the
corresponding results, the conclusions of
Theorem~\ref{th:main_th},
Corollary~\ref{cor:simultaneous_one_step_recovery},
Theorem~\ref{th:iter_contraction_random_observation},
Theorem~\ref{th:uniform_in_iteration_almost_exact}, and
Corollary~\ref{cor:iterative_exact_recovery}
remain valid for independent, not necessarily identically distributed,
permutation-valued corruptions satisfying this mean condition. This includes,
for example, corruption drawn uniformly from any transitive subgroup of
\(\Perm(m)\), such as the cyclic-shift subgroup.

For the spectral-initialization result, define, for \(i<j\),
\(
\boldsymbol{V}_{ij}:=\mP_i^\top\mU_{ij}\mP_j,
\)
\(
\boldsymbol{V}_{ji}:=\boldsymbol{V}_{ij}^\top.
\)
Suppose additionally that the gauged corruption array is block-exchangeable,
meaning that
\((\boldsymbol{V}_{\sigma(i)\sigma(j)})_{i\ne j}
\overset{\mathrm d}{=}(\boldsymbol{V}_{ij})_{i\ne j}\)
for every permutation \(\sigma\) of \([n]\).
Then the conclusions of Theorem~\ref{th:spectral_initialization} and
Corollary~\ref{cor:end_to_end_spectral_ppm} also remain valid. The mean
condition gives the same population matrix and spectral-norm bound, while
block exchangeability is used only in the final fixed-reference step that
transfers the global spectral-projector bound to the first reference block.
The uniform law automatically satisfies all of these requirements.

If the reference index is chosen uniformly from \([n]\), independently of
the measurements, the same spectral and end-to-end bounds hold without
block exchangeability. Probability is then taken over both the measurements
and the reference choice. The proof of
Theorem~\ref{th:spectral_initialization} gives this extension by averaging
the reference-column projector error over the chosen index.
\end{remark}

\begin{definition}[Normalized block error]
\label{def:delta_def}
For \(\bvP,\bvQ\in\Perm^n(m)\), define
\[
\delta(\bvP,\bvQ)
:=\frac{1}{n}
\left|\{i\in[n]:\mP_i\ne\mQ_i\}\right|,
\qquad
\delta_{\mathrm{glob}}(\bvP,\bvQ)
:=\min_{\mR\in\Perm(m)}\delta(\bvP,\bvQ\mR).
\]
\end{definition}

\begin{definition}[Read-\(k\) family]
Following~\cite{gavinsky2015tail}, random variables \(Z_1,\ldots,Z_N\)
form  a read-\(k\) family if 
they can be written as
\(
Z_i=f_i\bigl((X_\ell)_{\ell\in P_i}\bigr)
\),
\(i\in[N]
\),
where \(X_1,\ldots,X_M\) are mutually independent, the functions \(f_i\)
are measurable, and each index \(\ell\in[M]\) belongs to at most \(k\)
of the sets 
\(P_1,\ldots,P_N\subseteq[M]\).
 \end{definition}

\begin{assumption}[Per-block scaling]
\label{ass:one_step_scaling}
Suppose that \(\mL\) follows
Definition~\ref{def:uniform_corruption_model} with
\(m=m_n\), \(p=p_n\in(0,1]\), and
\(\pi_0=\pi_{0,n}\in(0,1]\). Assume that
\(
{\log m_n}/{\min\{m_n np_n\pi_{0,n}^2,\ np_n\pi_{0,n}\}}\to0.
\)
\end{assumption}

\begin{assumption}[Trajectory scaling]
\label{ass:uniform_corruption_scaling}
Suppose that \(\mL\) follows
Definition~\ref{def:uniform_corruption_model} with
\(m=m_n\), \(p=p_n\in(0,1]\), and
\(\pi_0=\pi_{0,n}\in(0,1]\). Assume that
\(
m_n/(np_n\pi_{0,n}^2)\to0.
\)
\end{assumption}

\begin{remark}
\label{rem:scaling_assumption_implies_limit}
Since \(m_n\ge2\), the per-block scaling implies
\(np_n\pi_{0,n}\to\infty\), and the trajectory scaling implies \(np_n\pi_{0,n}^2\to\infty\). 
Also,
\(\log m_n\le m_n\) and \(np_n\pi_{0,n}^2\le\min\{m_n np_n\pi_{0,n}^2,\ np_n\pi_{0,n}\}\), 
so the trajectory scaling implies the per-block scaling. Since \(\pi_{0,n}\le1\), the trajectory scaling also gives
\(
{m_n}/{(np_n)}
=
\pi_{0,n}^2
{m_n}/{(np_n\pi_{0,n}^2)}
\to0,
\)
and hence \(m_n/n\to0\).
\end{remark}

\begin{assumption}[Positive-overlap regime]
\label{ass:positive_overlap}
There exists a constant \(\epsilon\in(0,0.5)\), independent of \(n\), such that
\(
\delta_{\mathrm{glob}}(\bvP,\bvQ)
\le 0.5-\epsilon.
\)
Let
\(
\mR_\star
\in
\argmin_{\mR\in\Perm(m)}
\delta(\bvP,\bvQ\mR).
\)
Throughout the subsequent analysis, we work with the aligned representative
\(\bvQ\mR_\star\), which with a slight abuse of notation, we continue to
denote by \(\bvQ\). Thus,
\(
\delta(\bvP,\bvQ)
\le 0.5 - \epsilon.
\)
\end{assumption}
Thus, after optimal global alignment, an estimate in the positive-overlap
regime agrees exactly with the ground truth on at least
\((0.5+\epsilon)n\) blocks. The worst-case majority-gap bound has leading term \(np\pi_0\bigl(1-2\delta(\bvP,\bvQ)\bigr)\), 
apart from the omitted self-index. It flips sign at error 0.5.

\begin{assumption}[Independent estimate]
\label{ass:independent_estimates}
The estimate \(\bvQ\) is independent of all random variables used to generate \(\mL\); equivalently,
\(
\bvQ
\)
is independent of
\(\left\{
A_{ij},C_{ij},\mU_{ij}:
1\le i<j\le n
\right\}.
\)
\end{assumption}

\subsection{Supporting Lemmas}
\label{sec:supporting_lemmas}

Lemma~\ref{lem:signal_gap_random_observation} bounds the excess uncorrupted contribution from correctly estimated blocks in the one-step analysis. Scalar Bernstein bounds on matrix entries control the corrupted contribution in the one-step analysis and its estimate-independent part in the contraction analysis; Remark~\ref{rem:entrywise_sharpening} compares this approach with an operator-norm bound. Lemma~\ref{lem:sparse_noise_spectral_norm} is used in the contraction and spectral-initialization analyses, and Lemma~\ref{lem:read_two_concentration} in the contraction analysis.
\begin{lemma}[Concentration of the observed signal gap]
\label{lem:signal_gap_random_observation}
Let \(\mathcal I\) be a finite index set with
\(
|\mathcal I|=N,
\)
and let \(S\subseteq\mathcal I\) satisfy \(|S|>N/2\). For every
\(j\in\mathcal I\), let
\(
A_j\sim\operatorname{Bernoulli}(p),
\)
\(
C_j\sim\operatorname{Bernoulli}(\pi_0),
\)
and assume that all these random variables are mutually independent. Define
\(
J:=\{j\in\mathcal I:A_jC_j=1\},
\)
\(
\theta:=p\pi_0,
\)
and
\(
X:=|J\cap S|,
\)
\(
Y:=|J\setminus S|.
\)
Then, for every \(f>0\) satisfying
\(
(2|S|-N)\theta>f,
\)
we have
\[
\Prob(X-Y>f)
\ge
1-
\exp\left(
-\frac{\bigl((2|S|-N)\theta-f\bigr)^2}{2N\theta+\frac{2}{3}\bigl((2|S|-N)\theta-f\bigr)}
\right).
\]
\end{lemma}

\begin{proof}
Let $B_j:=A_jC_j$. Then the variables $\{B_j:j\in\mathcal I\}$ are independent
Bernoulli$(\theta)$ random variables. Setting $Z:=X-Y$, we have
\(
Z
=
\sum_{j\in S}B_j-\sum_{j\in\mathcal I\setminus S}B_j
\)
and hence
\(
\EE[Z]
=
|S|\theta-(N-|S|)\theta
=
(2|S|-N)\theta.
\)
Moreover,
\(
Z-\EE[Z]
=
\sum_{j\in S}(B_j-\theta)
-
\sum_{j\in\mathcal I\setminus S}(B_j-\theta).
\)
The summands in this expression are independent, mean-zero, bounded in
absolute value by $1$, and their total variance is
\(
\sum_{j\in\mathcal I}\Var(B_j)
=
N\theta(1-\theta)
\le N\theta.
\)
Therefore, Bernstein's inequality gives, for every $t>0$,
\(
\Prob\bigl(Z\le\EE[Z]-t\bigr)
\le
\exp\left(
-{t^2}/{(2N\theta+{2}t/{3})}
\right).
\)
Taking
\(
t = (2|S| - N) \theta - f > 0
\)
and then taking complements completes the proof.
\end{proof}

\begin{lemma}[Spectral norm of centered randomly observed blocks]
\label{lem:sparse_noise_spectral_norm}
Assume that \(1\le m\le n\). For \(1\le i<j\le n\), let
\(
A_{ij}\sim\operatorname{Bernoulli}(p)
\)
be independent observation indicators. Let
\(\mY_{ij}\in\mathbb R^{m\times m}\) be mutually independent random
matrices that are independent of the observation indicators and satisfy
\(
\|\mY_{ij}\|_2\le1,
\)
\(
\EE[\mY_{ij}]=\boldsymbol{\mu},
\)
where the common mean \(\boldsymbol{\mu}\) is symmetric.
Set
\(
\mL_{ij}:=A_{ij}\mY_{ij},
\)
\(
\mL_{ji}:=\mL_{ij}^{\top},
\)
and assume that the diagonal blocks \(\mL_{ii}\) are deterministic.
Let
\(
\overline{\mL}:=\EE[\mL].
\)
If
\(
np\ge C_0\log n
\)
for a sufficiently large absolute constant \(C_0>0\), then, for every
fixed \(\beta>0\), there exists a constant \(C_\beta>0\), independent of \(n\),
such that
\(
\|\mL-\overline{\mL}\|_2
\le
C_\beta\sqrt{np}
\)
with probability at least \(1-O(n^{-\beta})\).
\end{lemma}

\begin{proof}
By Jensen's inequality,
\(
\|\boldsymbol{\mu}\|_2
=
\|\EE[\mY_{ij}]\|_2
\le
\EE[\|\mY_{ij}\|_2]
\le1.
\)
Decompose
\(
\mL-\overline{\mL}
=
\mM^{(1)}+\mM^{(2)},
\)
where, for \(i<j\),
\(
\mM^{(1)}_{ij}
:=
A_{ij}(\mY_{ij}-\boldsymbol{\mu}),
\)
\(
\mM^{(2)}_{ij}
:=
(A_{ij}-p)\boldsymbol{\mu}.
\)
The lower-triangular blocks are defined by transposition, and the
diagonal blocks are zero.

The upper-triangular blocks of \(\mM^{(1)}\) are independent and
mean zero. Moreover, each block is zero whenever \(A_{ij}=0\), and
\(
\|\mY_{ij}-\boldsymbol{\mu}\|_2\le2.
\)
We apply Lemma~5.6 and Appendix~E of~\cite{chen2016projected} with block bound \(2\). That argument permits independent blocks with
dependent entries and uses only that their nonzero supports are dominated
by independent Bernoulli\((p)\) indicators. Its moment comparison introduces
the factor \(m^{1/(2\lceil\log n\rceil)}\), which is bounded by
\(e^{1/2}\) when \(m\le n\). Together with the row-sparsity bound, this
gives \(\EE\|\mM^{(1)}\|_2\le C\sqrt{np}\), with an absolute constant
\(C\). The convex concentration bound is centered at a median, which is at most
\(2\EE\|\mM^{(1)}\|_2\). With failure probability \(O(n^{-\beta})\),
it adds at most \(C_\beta\sqrt{\log n}\) to that median.
Since \(np\ge C_0\log n\), it follows that
\(
\|\mM^{(1)}\|_2\le C_\beta\sqrt{np}
\)
with probability at least \(1-O(n^{-\beta})\).

Let \(\mH\in\mathbb R^{n\times n}\) be the centered adjacency matrix
defined by
\(
H_{ij}:=A_{ij}-p,
\)
\(
i\neq j,
\)
\(
H_{ii}:=0.
\)
Since \(\boldsymbol{\mu}\) is symmetric,
\(
\mM^{(2)}
=
\mH\otimes\boldsymbol{\mu}.
\)
Applying Theorem~5.2 of~\cite{lei2015consistency} with its probability
parameter chosen larger than \(\beta\) gives
\(
\|\mH\|_2
\le
C_\beta\sqrt{np}
\)
with probability at least \(1-O(n^{-\beta})\). Hence,
\(
\|\mM^{(2)}\|_2
=
\|\mH\|_2\|\boldsymbol{\mu}\|_2
\le
C_\beta\sqrt{np}.
\)
The conclusion follows from the triangle inequality and a union bound.
\end{proof}

\begin{lemma}[Concentration for read-\(2\)]
\label{lem:read_two_concentration}
Let \(Z_1,\dots,Z_n\) be a read-\(2\) family of Bernoulli random
variables, and set
\(
q:=\sum_{i=1}^n\EE[Z_i]/n.
\)
If \(0<q<1\), then, for every \(r\in(q,1)\),
\[
\Prob\left(
\frac{1}{n}\sum_{i=1}^n Z_i\ge r
\right)
\le
\exp\left(
-\frac{n}{2}D_{\mathrm{KL}}(r\Vert q)
\right),
\]
where
\(
D_{\mathrm{KL}}(x\Vert y)
:=
x\log(x/y)+(1-x)\log((1-x)/(1-y)).
\)
Moreover, for every fixed \(\beta>0\), there exists a constant \(a_\beta>0\)
such that, for \(n\ge2\),
\[
\Prob\left(
\frac{1}{n}\sum_{i=1}^n Z_i
>
8q+a_\beta\frac{\log n}{n}
\right)
\le
n^{-\beta}.
\]
\end{lemma}

\begin{proof}
The first assertion is Theorem~1.1 of~\cite{gavinsky2015tail} with
\(k=2\). For the second assertion, set
\(
a_\beta:={2\beta}/{(\log8-1)}
\)
and
\(
r:=8q+a_\beta{\log n}/{n}.
\)
The result is immediate if \(q=0\) or \(r\ge1\). Otherwise,
\(0<q<r<1\), \(r/q\ge8\), and
\(
D_{\mathrm{KL}}(r\Vert q)
\ge
r\log(r/q)-r+q
\ge
r(\log8-1).
\)
Consequently,
\(
\frac{n}{2}D_{\mathrm{KL}}(r\Vert q)
\ge
\frac{a_\beta}{2}(\log8-1)\log n
=
\beta\log n.
\)
The first assertion therefore gives the claimed probability bound.
\end{proof}

\subsection{One-step recovery under independent estimates}
\label{sec:one_step_recovery}

Under Assumption~\ref{ass:independent_estimates}, conditioning on the estimate leaves the incident-edge variables independent. We control the uncorrupted signal gap and the corrupted contribution  separately to prove one-step blockwise recovery in the positive-overlap regime. A union bound then gives simultaneous recovery.

\begin{theorem}[Per-block one-step recovery from an independent estimate] \label{th:main_th}
Let  \(\bvP,\bvQ\in\Perm^n(m)\). Suppose that Assumptions~\ref{ass:one_step_scaling},
\ref{ass:positive_overlap}, and \ref{ass:independent_estimates} hold. In accordance with Assumption~\ref{ass:positive_overlap}, take \(\bvQ\) to be the aligned representative, so that
\(\delta(\bvP,\bvQ)=\delta_{\mathrm{glob}}(\bvP,\bvQ).\)
Define
\(\mR_i := \Proj\bigl((\mL\bvQ)_i\bigr),\)  \(1\le i\le n.\)
Then there exist constants \(C,c>0\), independent of \(n\) but possibly
depending on the fixed overlap margin \(\epsilon\), such that, for all
sufficiently large \(n\) and every \(1\le i\le n\),
\(
\Prob(\mR_i=\mP_i)
\ge
1
-
C\exp\bigl(-cnp\pi_0\bigr)
-
\)
\(C\exp\bigl(-c\min\{m\,np\pi_0^2,\,np\pi_0\}\bigr).\)

\end{theorem}
Assumption~\ref{ass:one_step_scaling} is used only to absorb the dimension
prefactor in the final step of the proof. Without this scaling assumption,
Definition~\ref{def:uniform_corruption_model} and
Assumptions~\ref{ass:positive_overlap} and
\ref{ass:independent_estimates} yield, for some constants
\(C,c,c_0>0\), independent of \(n\) but possibly depending on the fixed
overlap margin \(\epsilon\), and all sufficiently large \(n\),
\(
\Prob(\mR_i\neq\mP_i)
\le
C\exp(-cnp\pi_0)+
\)
\(2m^2\exp(-c_0\min\{m\,np\pi_0^2,\,np\pi_0\}).\)
Consequently, whenever 
\(m^2\exp(-c_0\min\{m\,np\pi_0^2,\,np\pi_0\})\to 0\), 
the per-block success probability tends to one.
Indeed, this condition also implies \(np\pi_0\to\infty\), so the first
exponential term vanishes as well. The stronger condition
\(
\log m=o(\min\{m\,np\pi_0^2,\,np\pi_0\})
\)
absorbs the factor \(m^2\), giving the dimension-free exponential bound in the theorem. For fixed \(m\) this condition is equivalent to \(np\pi_0^2\to\infty\);
it is weaker than \(\log m=o(np\pi_0^2)\) when \(m\) grows with \(n\). Since
\(\min\{m\,np\pi_0^2,\,np\pi_0\}\le np\pi_0\), the first exponential term in
the theorem is dominated by the second.

\begin{proof}
The alignment in Assumption~\ref{ass:positive_overlap} depends only on \(\bvQ\), so it preserves the independence in Assumption~\ref{ass:independent_estimates}. Fix \(i\in[n]\) and define
\(
\mathcal I_i:=[n]\setminus\{i\}.
\)
For \(j\in\mathcal I_i\), let \(A_{ij}\) and \(C_{ij}\) denote the
observation and non-corruption indicators, respectively. For \(j<i\), we use
the symmetric conventions
\(
A_{ij}:=A_{ji},
\)
\(
C_{ij}:=C_{ji},
\)
\(
\mU_{ij}:=\mU_{ji}^{\top}.
\)
Thus, for every \(j\in\mathcal I_i\),
\(
A_{ij}\sim\operatorname{Bernoulli}(p),
\)
\(
C_{ij}\sim\operatorname{Bernoulli}(\pi_0),
\)
and
\(
\mL_{ij}
=
A_{ij}
\left(
C_{ij}\mP_i\mP_j^\top
+
(1-C_{ij})\mU_{ij}
\right).
\)

Define
\(
J_i
:=
\{j\in\mathcal I_i:A_{ij}C_{ij}=1\}
\)
and
\(
K_i
:=
\{j\in\mathcal I_i:A_{ij}(1-C_{ij})=1\}.
\)
Thus, \(J_i\) contains the observed non-corrupted indices, whereas
\(K_i\) contains the observed corrupted indices. Notice that \(J_i\)
is defined using the latent non-corruption indicators rather than by the
equality \(\mL_{ij}=\mP_i\mP_j^\top\), since a uniformly random corrupted
permutation can coincide with the true permutation with positive probability.

Since \(\mL_{ii}=\boldsymbol{0}_{m\times m}\), 
\(
(\mL\bvQ)_i
=
\sum_{j\in\mathcal I_i}\mL_{ij}\mQ_j,
\)
so 
\(
(\mL\bvQ)_i
=
\mP_i\sum_{j\in J_i}\mP_j^\top\mQ_j
+
\sum_{j\in K_i}\mU_{ij}\mQ_j.
\)
Let
\(
S
:=
\{j\in[n]:\mQ_j=\mP_j\},
\)
\(
S_i:=S\cap\mathcal I_i.
\)
Separating the correctly and incorrectly estimated indices in \(J_i\)
gives
\begin{equation}
(\mL\bvQ)_i
=
|J_i\cap S_i|\mP_i
+
\mP_i\sum_{j\in J_i\setminus S_i}\mP_j^\top\mQ_j
+
\sum_{j\in K_i}\mU_{ij}\mQ_j.
\label{eq:mgm_i_block_3}
\end{equation}

\textbf{Intuition behind the proof.}
The first term in~\eqref{eq:mgm_i_block_3} is the true signal produced
by observed non-corrupted measurements connected to correctly estimated
indices. The second term contains observed non-corrupted measurements
connected to incorrectly estimated indices. The third term is the
sparse noise produced by observed corrupted measurements.

Set \(f(n):=\epsilon np\pi_0/2.\) Thus, \(f(n)\) is a fixed positive fraction of the signal scale. We will show that the observed non-corrupted contribution has projection gap larger than \(f(n)d_i(\mR)\), uniformly over all competing permutations, while the centered corrupted contribution is smaller than this gap.

\noindent\textbf{Step 1: Decomposition and centering.}
Define
\(
\mA_i
:=
|J_i\cap S_i|\mP_i
+
\mP_i\sum_{j\in J_i\setminus S_i}
\mP_j^\top\mQ_j
\)
and
\(
\mB_i
:=
\sum_{j\in K_i}\mU_{ij}\mQ_j.
\)
Then
\(
(\mL\bvQ)_i=\mA_i+\mB_i.
\)
Let
\(
\rho:=p(1-\pi_0).
\)
For each \(j\), the corrupted contribution is present with probability
\(\rho\), and
\(
\EE\left[
A_{ij}(1-C_{ij})\mU_{ij}\mQ_j
\,\middle|\,\bvQ
\right]
=
\frac{\rho}{m}\vone\vone^\top.
\)
Define the centered corrupted contribution
\(
\widetilde{\mB}_i
:=
\mB_i-\frac{(n-1)\rho}{m}\vone\vone^\top.
\)

\noindent\textbf{Step 2: A sufficient signal-gap condition.}
Set
$
X_i:=|J_i\cap S_i|,
$
$
Y_i:=|J_i\setminus S_i|.
$
Fix $\mR\in\Perm(m)$ with $\mR\neq\mP_i,$ and define
$
d_i(\mR)
:=
\langle\mP_i,\mP_i-\mR\rangle
=
m-\langle\mP_i,\mR\rangle.
$
Because two distinct permutation matrices differ in at least two rows,
$
d_i(\mR)\ge2.
$
For every \(\mV\in\Perm(m)\),
$
\langle\mV,\mP_i-\mR\rangle
\ge
-d_i(\mR).
$
For \(j\in J_i\setminus S_i\), the matrix
$
\mV_{ij}:=\mP_i\mP_j^\top\mQ_j
$
is a permutation matrix. Consequently,
\[
\langle\mA_i,\mP_i-\mR\rangle
=
X_i d_i(\mR)
+
\sum_{j\in J_i\setminus S_i}
\langle\mV_{ij},\mP_i-\mR\rangle
\ge
X_i d_i(\mR)-Y_i d_i(\mR)
=
(X_i-Y_i)d_i(\mR).
\]
It follows that, on the event
\(
X_i-Y_i>f(n),
\)
we have, simultaneously for every
\(\mR\in\Perm(m)\setminus\{\mP_i\}\),
\(
\left\langle
\mA_i,\mP_i-\mR
\right\rangle
>
f(n)d_i(\mR).
\)
Since
\(
\left\langle
\vone\vone^\top,\mP_i-\mR
\right\rangle
=0,
\)
we also have
\(
\left\langle
(\mL\bvQ)_i,\mP_i-\mR
\right\rangle
=
\left\langle
\mA_i,\mP_i-\mR
\right\rangle
+
\left\langle
\widetilde{\mB}_i,\mP_i-\mR
\right\rangle.
\)

\noindent\textbf{Step 3: Uniform control of the corrupted contribution.}

For \(\mR\in\Perm(m)\setminus\{\mP_i\}\), the matrix \(\mP_i-\mR\) has
exactly \(2d_i(\mR)\) nonzero entries, each equal to \(\pm1\): in each of the
\(d_i(\mR)\) rows in which \(\mP_i\) and \(\mR\) differ, it has one entry
\(+1\) and one entry \(-1\), and its remaining rows vanish. Writing
\(\|\cdot\|_{\max}\) for the largest absolute entry of a matrix, we therefore
have
\(
\left|
\left\langle
\widetilde{\mB}_i,\mP_i-\mR
\right\rangle
\right|
\le
2d_i(\mR)\,\|\widetilde{\mB}_i\|_{\max}
\)
simultaneously for every \(\mR\neq\mP_i\). Thus,
\(\Proj((\mL\bvQ)_i)=\mP_i\) whenever
\(X_i-Y_i>f(n)\) and
\(\|\widetilde{\mB}_i\|_{\max}<f(n)/2\).

Conditionally on \(\bvQ\), we have
\(
\widetilde{\mB}_i=\sum_{j\in\mathcal I_i}\mS_j,
\)
\(
\mS_j:=A_{ij}(1-C_{ij})\mU_{ij}\mQ_j-\frac{\rho}{m}\vone\vone^\top,
\)
and the summands \(\mS_j\), \(j\in\mathcal I_i\), are independent, because
they are determined by the variables attached to the distinct unordered pairs
\(\{i,j\}\). Fix an entry \((k,\ell)\). Since \(\mU_{ij}\mQ_j\) is a permutation
matrix with
\(
\EE[\mU_{ij}\mQ_j\mid\bvQ]
=
m^{-1}\vone\vone^\top\mQ_j
=
m^{-1}\vone\vone^\top,
\)
the scalar \((\mU_{ij}\mQ_j)_{k\ell}\in\{0,1\}\) has conditional mean \(1/m\) and
is independent of \(A_{ij}(1-C_{ij})\sim\operatorname{Bernoulli}(\rho)\).
Hence \((\mS_j)_{k\ell}\) is mean zero, bounded in absolute value by one, and has
variance at most
\(
\EE\bigl[A_{ij}(1-C_{ij})(\mU_{ij}\mQ_j)_{k\ell}\,\big|\,\bvQ\bigr]=\rho/m.
\)
Bernstein's inequality, applied to the \(n-1\) independent summands
\((\mS_j)_{k\ell}\) with \(t=f(n)/2\), followed by a union bound over the
\(m^2\) entries, gives
\[
\Prob\left(
\|\widetilde{\mB}_i\|_{\max}\ge\frac{f(n)}{2}
\,\middle|\,\bvQ
\right)
\le
2m^2\exp\left(
-\frac{f(n)^2/4}{2\bigl((n-1)\rho/m+f(n)/6\bigr)}
\right).
\]
We bound the exponent from below. If \(2(n-1)\rho/m\ge f(n)/3\), the
denominator is at most \(4(n-1)\rho/m\le4np/m\), and the exponent is at least
\(f(n)^2m/(16np)=\epsilon^2m\,np\pi_0^2/64\); otherwise the denominator is at
most \(2f(n)/3\), and the exponent is at least
\(3f(n)/8=3\epsilon\,np\pi_0/16\). Hence
\(
\Prob\left(
\|\widetilde{\mB}_i\|_{\max}\ge {f(n)}/{2}
\,\middle|\,\bvQ
\right)
\le
2m^2\exp\bigl(-c_0\min\{m\,np\pi_0^2,\,np\pi_0\}\bigr),
\)
where \(c_0>0\) depends only on \(\epsilon\). Since this bound is uniform in
\(\bvQ\), it also holds unconditionally.
Assumption~\ref{ass:one_step_scaling} gives
\(\log m=o(\min\{m\,np\pi_0^2,\,np\pi_0\})\). Therefore, after decreasing the
exponent constant, for all sufficiently large \(n\),
\(
\Prob\left(
\|\widetilde{\mB}_i\|_{\max}\ge{f(n)}/{2}
\right)
\le
C\exp\bigl(-c\min\{m\,np\pi_0^2,\,np\pi_0\}\bigr).
\)

\noindent\textbf{Step 4: Concentration of the signal gap.}
Let
\(
\alpha:=1-\delta(\bvP,\bvQ),
\)
\(
\theta:=p\pi_0.
\)
Then \(|S|=\alpha n\). Since \(S_i=S\cap\mathcal I_i\),
\(
|S_i|\ge |S|-1
\)
and hence
\(
2|S_i|-(n-1)
\ge
(2\alpha-1)n-1.
\)
Assumption~\ref{ass:positive_overlap} implies
\(
2\alpha-1\ge 2\epsilon.
\)
Therefore, for all sufficiently large \(n\),
\(
2|S_i|-(n-1)\ge \epsilon n>0.
\)

By Assumption~\ref{ass:independent_estimates}, conditionally on \(\bvQ\),
the set \(S_i\) is fixed and independent of
\(\{A_{ij}C_{ij}:j\in\mathcal I_i\}\). We may therefore apply
Lemma~\ref{lem:signal_gap_random_observation} with
\(
\mathcal I=\mathcal I_i,
\)
\(
N=n-1,
\)
\(
S=S_i,
\)
\(
f=f(n).
\)
Moreover,
\({f(n)}/{(n\theta)}=\epsilon/2.\) 
Since
\(2|S_i|-(n-1)\ge\epsilon n\)
for all sufficiently large \(n\), we have
\(
\bigl(2|S_i|-(n-1)\bigr)\theta-f(n)
\ge
0.5 \epsilon n\theta.
\)
Lemma~\ref{lem:signal_gap_random_observation} therefore gives
\(\Prob\left(
X_i-Y_i\le f(n)
\,\middle|\,\bvQ
\right)
\le
\exp(-cn\theta),\)
where \(c>0\) may depend on \(\epsilon\). Since this bound is uniform over
all estimates satisfying Assumptions~\ref{ass:positive_overlap} and
\ref{ass:independent_estimates}, it also holds unconditionally:
\begin{equation}
\label{eq:signal_asymptotic_preliminary}
\Prob\left(X_i-Y_i\le f(n)\right)
\le
\exp(-cn\theta).
\end{equation}

\noindent\textbf{Conclusion.}
Combining the entrywise estimate from Step~3 with
\eqref{eq:signal_asymptotic_preliminary}, we obtain
\[
\Prob\left(
\Proj\bigl((\mL\bvQ)_i\bigr)\neq\mP_i
\right)
\le
C\exp\bigl(-cnp\pi_0\bigr)
+
C\exp\bigl(-c\min\{m\,np\pi_0^2,\,np\pi_0\}\bigr).
\]
This completes the proof.
\end{proof}

\begin{remark}[Comparison with the operator-norm bound]
\label{rem:entrywise_sharpening}
The corrupted contribution in Step~3 can alternatively be controlled in
operator norm. Since \(\mP_i-\mR\) has rank at most \(d_i(\mR)\) and squared
Frobenius norm \(2d_i(\mR)\), its nuclear norm is at most
\(\sqrt2\,d_i(\mR)\), so
\(
\left|\left\langle\widetilde{\mB}_i,\mP_i-\mR\right\rangle\right|
\le
\sqrt{2}\,\|\widetilde{\mB}_i\|_2\,d_i(\mR).
\)
Conditionally on \(\bvQ\), the summands \(\mS_j\) have spectral norm at most
one and matrix variance statistic at most \((n-1)\rho\), so the rectangular
matrix Bernstein inequality~\cite[Theorem~6.1.1]{tropp2015introduction} with
\(t=f(n)/\sqrt2\) gives
\(
\Prob\bigl(\|\widetilde{\mB}_i\|_2\ge f(n)/\sqrt2\,\big|\,\bvQ\bigr)
\le
2m\exp(-c_0np\pi_0^2).
\)
This route yields the conclusion of Theorem~\ref{th:main_th} with the term
\(C\exp(-c\min\{m\,np\pi_0^2,\,np\pi_0\})\) replaced by
\(C\exp(-cnp\pi_0^2)\), under the stronger condition \(\log m=o(np\pi_0^2)\).
The operator-norm bound controls fluctuations on the scale \(\sqrt{np}\), 
while the entrywise variance bound has the finer scale
\(np(1-\pi_0)/m\). Pairing the former with \(\mP_i-\mR\) 
discards this dimension gain, even though the difference has only \(2d_i(\mR)\) nonzero entries. The entrywise bound is used again for the set
\(\mathcal B_{\mathrm{ind}}\) in the proof of
Theorem~\ref{th:iter_contraction_random_observation}. The proof controls the 
data-dependent error sets and patterns defining \(\mathcal B_{\mathrm{sig}}\) and \(\mathcal B_{\mathrm{dep}}\) through spectral and Frobenius norms. The Frobenius bound introduces the factor \(m\) in the contraction bound.
\end{remark}

\begin{corollary}[Simultaneous one-step recovery]
\label{cor:simultaneous_one_step_recovery}
Under the assumptions of Theorem~\ref{th:main_th},
\[
\Prob\left(
\mR_i=\mP_i\text{ for every }i\in[n]
\right)
\ge
1
-
Cn\exp(-cnp\pi_0)
-
Cn\exp\bigl(-c\min\{m\,np\pi_0^2,\,np\pi_0\}\bigr).
\]
Consequently, since \(\min\{m\,np\pi_0^2,\,np\pi_0\}\le np\pi_0\), simultaneous exact recovery
holds with probability tending to one whenever
\(
c\min\{m\,np\pi_0^2,\,np\pi_0\}-\log n\to\infty,
\)
and in particular whenever \(cnp\pi_0^2-\log n\to\infty\).
\end{corollary}
This follows immediately by applying Theorem~\ref{th:main_th} for every
\(i\in[n]\) and taking a union bound. The per-block theorem requires only
\(\log m=o(\min\{m\,np\pi_0^2,\,np\pi_0\})\), whereas simultaneous recovery introduces the
additional \(\log n\) cost.

\subsection{Basin-Uniform Contraction for Data-Dependent Estimates}
\label{sec:basin_uniform_contraction}

The event below is determined by the random draw and holds for every
estimate in the positive-overlap basin. The estimate may depend on the
measurements. Basin invariance then allows this event to control all PPM
iterates.

\begin{theorem}[Uniform contraction over the positive-overlap basin]
\label{th:iter_contraction_random_observation}
Let \(\mP_i\in\Perm(m)\), \(1\le i\le n\), and set
\(\bvP=[\mP_1^\top,\ldots,\mP_n^\top]^\top\). Suppose that
Assumption~\ref{ass:uniform_corruption_scaling} holds.
Assume in addition that
\(
np\ge C_0\log n
\)
for a sufficiently large absolute constant \(C_0>0\).

Fix
\(
\epsilon\in(0,0.5),
\)
\(
\beta>0.
\)
Then there exist constants \(c,C_1,C_2>0\), independent of \(n\) but possibly
depending on \(\epsilon\) and \(\beta\), such that, for all sufficiently large
\(n\), with probability at least
\(
1-O(n^{-\beta}),
\)
the following holds simultaneously for every block vector

\(
\bvQ
=
[(\mQ_1)^\top,\ldots,(\mQ_n)^\top]^\top,
\)
where
\(
\mQ_i\in\Perm(m),
\)
satisfying
\(
\delta
:=
\delta_{\mathrm{glob}}(\bvP,\bvQ)
\le
0.5-\epsilon.
\)
Let
\(
\bvQ^{+}
=
\Proj\bigl(\mL\bvQ\bigr),
\)
where the projection is applied blockwise, and define
 \(
\delta^{+}
:=
\delta_{\mathrm{glob}}(\bvP,\bvQ^{+}).
\)
Then
\[
\delta^{+}\le\gamma_n\delta+\xi_n,
\]
where
\begin{equation}
\label{eq:contraction_parameters}
\gamma_n:=C_1\frac{m}{np\pi_0^2}, \qquad
\xi_n:=C_2\left[
\exp\bigl(-c\min\{m\,np\pi_0^2,\,np\pi_0\}\bigr)
+\frac{\log n}{n}
\right].
\end{equation}
\end{theorem}

\begin{proof}
For brevity, write
\(
\theta:=p\pi_0
\)
and
\(
\rho:=p(1-\pi_0).
\)
Choose \(\mR_\star\in\Perm(m)\) that minimizes
\(\delta\bigl(\bvP,\bvQ\mR\bigr)\)
over
\(\mR\in\Perm(m)\).
For every \(\mA\in\mathbb R^{m\times m}\),
\(\mR\in\Perm(m)\), and \(\mQ\in\Perm(m)\),
\[
\mQ\in
\operatorname*{argmax}_{\mP\in\Perm(m)}
\langle\mP,\mA\rangle
\quad\Longleftrightarrow\quad
\mQ\mR\in
\operatorname*{argmax}_{\mP\in\Perm(m)}
\langle\mP,\mA\mR\rangle.
\]
Consequently, after replacing  \(\bvQ\) by \(\bvQ\mR_\star\), we may replace the next iterate by \(\bvQ^{+}\mR_\star\); each of its blocks remains a valid projection of the corresponding aligned score matrix. The inequalities below establish uniqueness for every block declared correctly recovered. We may therefore work with these aligned representatives and define
\(S := \left\{j\in[n]:\mQ_j=\mP_j \right\},\)
\(T:=S^c.\)
Thus, \(|T|=n\delta.\)
By the assumption on \(\delta\), \(|T| \le \left(0.5-\epsilon\right)n.\)

For \(1\le i<j\le n\), let \(A_{ij}\) and \(C_{ij}\) be the observation
and non-corruption indicators, respectively. For \(j<i\), use the symmetric
conventions
\(
A_{ij}:=A_{ji},
\)
\(
C_{ij}:=C_{ji},
\)
\(
\mU_{ij}:=\mU_{ji}^{\top}.
\)
Thus, for every \(i\neq j\),
\(
A_{ij}\sim\operatorname{Bernoulli}(p),
\)
\(
C_{ij}\sim\operatorname{Bernoulli}(\pi_0).
\)
Set
\(
G_{ij}:=A_{ij}C_{ij}.
\)
When \(A_{ij}=1\) and \(C_{ij}=0\), write the corrupted measurement as
\(\mU_{ij}\), and define
\(
\mW_{ij}:=A_{ij}(1-C_{ij})\mU_{ij}.
\)
We take \(G_{ii}=0\) and \(\mW_{ii}=\vzero\).

Set \(f:=\epsilon np\pi_0/2=\epsilon n\theta/2.\) Since \(\epsilon\) is fixed, \(f\) is a  constant fraction of the expected uncorrupted degree. This gives the contraction scale \(m/(np\pi_0^2)\).
For each \(i\), let
\(
X_i:=\sum_{j\in S}G_{ij},
\)
\(
Y_i:=\sum_{j\in T}G_{ij}.
\)
These are, respectively, the numbers of observed non-corrupted
measurements in row \(i\) coming from correctly and incorrectly
estimated blocks.

We first control the set
\(
\mathcal B_{\mathrm{sig}}
:=
\left\{i\in[n]:X_i-Y_i\le f\right\}.
\)
Let
\(
\nu_i:=X_i+Y_i=\sum_{j\neq i}G_{ij}
\)
and define
\(
\mathcal B_{\mathrm{deg}}
:=
\left\{
i\in[n]:
\nu_i<
\left(1-\frac{\epsilon}{2}\right)(n-1)\theta
\right\}.
\)
Since
\(\nu_i\sim\operatorname{Binomial}(n-1,\theta),\) the Chernoff bound 
gives
\(
\Prob(i\in\mathcal B_{\mathrm{deg}})
\le
\exp(-cn\theta).
\)
For each \(i\in[n]\), define
\(
Z_i^{\mathrm{deg}}
:=
\boldsymbol{1}_{\{i\in\mathcal B_{\mathrm{deg}}\}}.
\)
The variables
\(
Z_1^{\mathrm{deg}},\dots,Z_n^{\mathrm{deg}}
\)
form a read-\(2\) family with respect to the independent edge variables
\(
\{G_{uv}:1\le u<v\le n\},
\)
because \(G_{uv}\) can affect only
\(Z_u^{\mathrm{deg}}\) and \(Z_v^{\mathrm{deg}}\). 
Set
\(
q_{\mathrm{deg}}
:=
\frac{1}{n}\sum_{i=1}^n\EE[Z_i^{\mathrm{deg}}].
\)
The preceding pointwise bound gives
\(
q_{\mathrm{deg}}\le\exp(-cn\theta).
\)
The final assertion of
Lemma~\ref{lem:read_two_concentration} therefore gives
\[
\frac{|\mathcal B_{\mathrm{deg}}|}{n}
\le
C\left[
\exp(-cn\theta)+\frac{\log n}{n}
\right]
\text{ except on an event of probability at most } n^{-\beta}.
\]

Now suppose that \(i\notin\mathcal B_{\mathrm{deg}}\) and
\(i\in\mathcal B_{\mathrm{sig}}\). Since
\(
X_i-Y_i=\nu_i-2Y_i\le f,
\)
we have
\(
Y_i\ge\frac{\nu_i-f}{2}.
\)
Define the centering term
\(
\overline{Y}_i
:=
\theta\left(|T|-\mathbf 1_{\{i\in T\}}\right).
\)
Then
\(
\overline{Y}_i\le\theta|T|.
\)
Although \(T\) may depend on \(\mG\), the following identity holds
pointwise:
\(
Y_i-\overline{Y}_i
=
\left((\mG-\EE[\mG])\vone_T\right)_i,
\)
where \(\mG=(G_{ij})_{i,j=1}^n\).

Using
\(
|T|\le\left(0.5-\epsilon\right)n,
\)
\(
\nu_i\ge
\left(1-\frac{\epsilon}{2}\right)(n-1)\theta,
\)
and 
\(f=\epsilon n\theta/2\), we obtain, for all sufficiently large \(n\),
\[
Y_i-\overline{Y}_i
\ge
\left(
\frac{1}{2}\left(1-\frac{\epsilon}{2}\right)
-\frac{\epsilon}{4}
-\left(0.5-\epsilon\right)
-o(1)
\right)n\theta
\ge
c\epsilon n\theta,
\]
Consequently,
\(
\left|
\left((\mG-\EE[\mG])\vone_T\right)_i
\right|
\ge
c\epsilon n\theta.
\)

Apply Lemma~\ref{lem:sparse_noise_spectral_norm} with block size one,
observation indicators \(A_{ij}\), and scalar blocks 
\(C_{ij}\).
Since \(\EE[C_{ij}]=\pi_0,\) the lemma gives, with probability at least \(1-O(n^{-\beta})\),
\(
\|\mG-\EE[\mG]\|_2^2\le Cnp.
\)
Therefore,
\[
\bigl|
\mathcal B_{\mathrm{sig}}
\setminus\mathcal B_{\mathrm{deg}}
\bigr|
(c\epsilon n\theta)^2
\le
\|(\mG-\EE[\mG])\vone_T\|_2^2
\le
\|\mG-\EE[\mG]\|_2^2\|\vone_T\|_2^2
\le
Cnp|T|.
\]
Since \(\theta=p\pi_0\), it follows that
\(
\left|
\mathcal B_{\mathrm{sig}}
\setminus\mathcal B_{\mathrm{deg}}
\right|
\le
C{|T|}/{(np\pi_0^2)}.
\)

We next control the corrupted contribution. Define
 \[
\mE_j:=\mQ_j-\mP_j,
\qquad
(\bvE_T)_j
:=
\begin{cases}
\mE_j, & j\in T,\\
\vzero, & j\in S.
\end{cases}
\]
For each \(i\), write
\(\sum_{j\neq i}\mW_{ij}\mQ_j = \rho(n-1)\mJ+\mN_i+\mD_i,\)
\(\mJ:=\frac{1}{m}\vone\vone^\top,\)
where
\(
\mD_i
:=
\sum_{j\in T}\mW_{ij}\mE_j
\)
and
\(
\mN_i
:=
\sum_{j\neq i}
\left(\mW_{ij}\mP_j-\rho\mJ\right).
\)
Because
\(
\mJ\mE_j=\vzero,
\)
we may write the stacked dependent term as
\(
\bvD=(\mW-\EE[\mW])\bvE_T.
\)

Since the difference of two permutation matrices satisfies
\(\|\mQ_j-\mP_j\|_F^2\le 2m,\)
we have
\(\|\bvE_T\|_F^2\le 2m|T|.\)
For this application, write
\(\mW_{ij} = A_{ij}\mY_{ij},\)
\(\mY_{ij} := (1-C_{ij})\mU_{ij}.\)
The blocks \(\mY_{ij}\) are independent of the observation indicators,
have spectral norm at most one, and have common symmetric mean
\(
\EE[\mY_{ij}]
=
(1-\pi_0)\mJ.
\)
Therefore, Lemma~\ref{lem:sparse_noise_spectral_norm} gives, with
probability at least \(1-O(n^{-\beta})\),
\(
\|\mW-\EE[\mW]\|_2^2\le Cnp.
\)
On this event,
\[
\sum_{i=1}^n\|\mD_i\|_F^2
=
\|(\mW-\EE[\mW])\bvE_T\|_F^2
\le
\|\mW-\EE[\mW]\|_2^2\|\bvE_T\|_F^2
\le
Cnp\,m|T|.
\]

Define
\(
\mathcal B_{\mathrm{dep}}
:=
\left\{
i\in[n]:
\|\mD_i\|_F\ge{f}/{2}
\right\}.
\)
Every \(i\in\mathcal B_{\mathrm{dep}}\) contributes at least \(f^2/4\)
to the preceding sum. Hence,
\(
|\mathcal B_{\mathrm{dep}}|
{f^2}/{4}
\le
\sum_{i=1}^n\|\mD_i\|_F^2
\le
Cnp\,m|T|.
\)
Since \(f^2=\epsilon^2n^2p^2\pi_0^2/4,\) we obtain \(|\mathcal B_{\mathrm{dep}}| \le C{m}|T|/(np\pi_0^2),\) where the constant may depend on \(\epsilon\).

It remains to control \(\mN_i\). Define
\(
\mathcal B_{\mathrm{ind}}
:=
\left\{
i\in[n]:
\|\mN_i\|_{\max}\ge f/4
\right\},
\)
where \(\|\cdot\|_{\max}\) denotes the largest absolute entry.
For every fixed \(i\), the matrices
\(\{\mW_{ij}\mP_j:j\neq i\}\) are independent. Each is zero with probability
\(1-\rho\), and, conditional on being nonzero, is a permutation matrix with
mean \(m^{-1}\vone\vone^\top\), since
\(
\EE[\mU_{ij}\mP_j]
=
m^{-1}\vone\vone^\top\mP_j
=
m^{-1}\vone\vone^\top.
\)
Hence every entry of \(\mN_i\) is a sum of \(n-1\) independent
mean-zero summands \((\mW_{ij}\mP_j)_{k\ell}-\rho/m\), each bounded in
absolute value by one and with variance at most \(\rho/m\). As in Step~3 of
the proof of Theorem~\ref{th:main_th}, Bernstein's inequality with
\(t=f/4\) and a union bound over the \(m^2\) entries give
\[
\Prob(i\in\mathcal B_{\mathrm{ind}})
\le
2m^2\exp\left(
-\frac{f^2/16}{2\bigl((n-1)\rho/m+f/12\bigr)}
\right)
\le
2m^2\exp\bigl(-c_0\min\{m\,np\pi_0^2,\,np\pi_0\}\bigr),
\]
where \(c_0>0\) depends only on \(\epsilon\); the second inequality follows
from the case analysis in that proof, the exponent being at least
\(f^2m/(64np)=\epsilon^2m\,np\pi_0^2/256\) when \(2(n-1)\rho/m\ge f/6\) and
at least \(3f/16=3\epsilon\,np\pi_0/32\) otherwise.
Assumption~\ref{ass:uniform_corruption_scaling} implies
\(\log m=o(np\pi_0^2)\), and \(np\pi_0^2\le\min\{m\,np\pi_0^2,\,np\pi_0\}\).
Hence, after decreasing the exponent constant,
\[
\Prob(i\in\mathcal B_{\mathrm{ind}})
\le
C\exp\bigl(-c\min\{m\,np\pi_0^2,\,np\pi_0\}\bigr).
\]

For each \(i\in[n]\), define
\(
Z_i^{\mathrm{ind}}
:=
\boldsymbol{1}_{\{i\in\mathcal B_{\mathrm{ind}}\}}.
\)
The variables
\(
Z_1^{\mathrm{ind}},\dots,Z_n^{\mathrm{ind}}
\)
form a read-\(2\) family with respect to the independent unordered-edge
variables
\(
\{\mW_{uv}:1\le u<v\le n\},
\)
because \(\mW_{uv}\) can affect only
\(Z_u^{\mathrm{ind}}\) and \(Z_v^{\mathrm{ind}}\).

Set
\(
q_{\mathrm{ind}}
:=
\frac{1}{n}\sum_{i=1}^n\EE[Z_i^{\mathrm{ind}}].
\)
Then
\(
q_{\mathrm{ind}}
\le
\)
\(C\exp\bigl(-c\min\{m\,np\pi_0^2,\,np\pi_0\}\bigr).\)
By Lemma~\ref{lem:read_two_concentration}, with probability at least \(1-n^{-\beta}\),
\[
\frac{|\mathcal B_{\mathrm{ind}}|}{n}
\le C 
\left[
\exp\bigl(-c\min\{m\,np\pi_0^2,\,np\pi_0\}\bigr)
+
\frac{\log n}{n}
\right]
\]

We now show that every index outside
\(
\mathcal B_{\mathrm{sig}}
\cup
\mathcal B_{\mathrm{dep}}
\cup
\mathcal B_{\mathrm{ind}}
\)
is recovered correctly. Fix such an index \(i\), let
\(\mR\in\Perm(m)\setminus\{\mP_i\}\), and write \(d_i(\mR):=m-\langle\mP_i,\mR\rangle\).
Since two distinct permutation matrices differ in at least two rows,
\(
d_i(\mR)\ge2
\)
and
\(
\|\mP_i-\mR\|_F=\sqrt{2d_i(\mR)}.
\)
For every \(\mV\in\Perm(m)\),
\(
\left\langle
\mV,\mP_i-\mR
\right\rangle
\ge
-d_i(\mR).
\)
Therefore, the observed non-corrupted contribution satisfies
 \(
\left\langle
\sum_{j\neq i}
G_{ij}\mP_i\mP_j^\top\mQ_j,
\mP_i-\mR
\right\rangle
\ge
(X_i-Y_i)d_i(\mR)
>
f\,d_i(\mR).
\)

As shown in Step~3 of the proof of Theorem~\ref{th:main_th},
\(\mP_i-\mR\) has exactly \(2d_i(\mR)\) nonzero entries, each equal to
\(\pm1\). Since \(i\notin\mathcal B_{\mathrm{ind}}\), it follows that
\(
\left|
\left\langle
\mN_i,\mP_i-\mR
\right\rangle
\right|
\le
2d_i(\mR)\,\|\mN_i\|_{\max}
<
\frac{f}{2}d_i(\mR).
\)
Since \(i\notin\mathcal B_{\mathrm{dep}}\),
\(
\left|
\left\langle
\mD_i,\mP_i-\mR
\right\rangle
\right|
\le
\|\mD_i\|_F\|\mP_i-\mR\|_F
<
{f}\sqrt{2d_i(\mR)}/{2}
\le
{f}d_i(\mR)/{2},
\)
where the last inequality uses \(d_i(\mR)\ge2\). Finally,
\(
\left\langle
\mJ,\mP_i-\mR
\right\rangle
=0.
\)
Combining these estimates gives
\(\left\langle (\mL\bvQ)_i, \mP_i-\mR \right\rangle > 0\)
for every \(\mR\neq\mP_i\). Hence,
\(\mQ_i^{+} = \Proj\bigl((\mL\bvQ)_i\bigr) = \mP_i.\)
Since optimizing over the global alignment can only decrease the block
error, it follows that
\[
n\delta^{+}
\le
|\mathcal B_{\mathrm{sig}}|
+
|\mathcal B_{\mathrm{dep}}|
+
|\mathcal B_{\mathrm{ind}}|
\le
C_1\frac{m}{np\pi_0^2}|T|
+
C_2n
\left[
\exp\bigl(-c\min\{m\,np\pi_0^2,\,np\pi_0\}\bigr)
+
\frac{\log n}{n}
\right],
\]
where the \(\mathcal B_{\mathrm{deg}}\) term was bounded using
\(\exp(-cnp\pi_0)\le\exp(-c\min\{m\,np\pi_0^2,\,np\pi_0\})\).
Because 
\(|T|=n\delta\), dividing by \(n\) yields
\(\delta^{+} \le C_1{m}\delta/(np\pi_0^2) +\)
\(C_2[e^{-c\min\{m\,np\pi_0^2,\,np\pi_0\}}+{\log n}/{n}]\).

The two read-\(2\) concentration events controlling \(\mathcal B_{\mathrm{deg}}\) and \(\mathcal B_{\mathrm{ind}}\), together with the two spectral-norm events, are determined by the random draw and do not depend on \(\bvQ\). On their intersection, the deterministic bounds apply to every error set \(T\) and corresponding error vector \(\bvE_T\), and hence to every estimate in the positive-overlap regime.  This intersection has probability at least \(1-O(n^{-\beta}).\)
\end{proof}

\begin{remark}
The half-error boundary cannot be included:
with complete noiseless observations and identity ground truth, an
initialization split equally between the identity and a transposition
retains aligned error \(1/2\) at every iteration. The proof is immediate.   
\end{remark}

\paragraph{Dependence on the basin margin}
For fixed \(\epsilon\in(0,0.5)\), the constants may be chosen so that
\(
\gamma_n
\le
C_\beta\epsilon^{-2}{m}/{(np\pi_0^2)}
\)
and \(\xi_n \le \) \(C_\beta[e^{-c\epsilon^2\min\{m\,np\pi_0^2,\,np\pi_0\}}+{\log n}/{n}]\). Here \(c>0\) is absolute and \(C_\beta>0\) depends on \(\beta\), but neither constant depends on \(\epsilon\). Thus, the contraction guarantee deteriorates at order \(\epsilon^{-2}\) as the \(0.5\) block-error boundary is approached. This reflects that the expected uncorrupted score gap is of order \(\epsilon np\pi_0\). The theorem keeps \(\epsilon\) fixed and makes no uniform claim as \(\epsilon\downarrow0\). 
Assumption~\ref{ass:uniform_corruption_scaling} enters the proof only through  the eventual bound \(m\le n\) in Lemma~\ref{lem:sparse_noise_spectral_norm} and through the absorption of the prefactor \(2m^2\) in the bound for \(\mathcal B_{\mathrm{ind}}\); the  multiplicative factor is therefore smaller than one as soon as \(m\le c_\beta\epsilon^2np\pi_0^2\) for a sufficiently small constant \(c_\beta>0\). This condition concerns the multiplicative factor; basin invariance also requires \(\gamma_n(1/2-\epsilon)+\xi_n\le1/2-\epsilon\).

\subsection{Uniform-in-Iteration Almost-Exact and Finite-Step Exact Recovery}
\label{sec:contraction_across_iterations}

Fix \(\beta>0\) and \(\epsilon\in(0,0.5)\), and let the constants below
be those from Theorem~\ref{th:iter_contraction_random_observation}.
Let
\(
\bvQ^{(k+1)}
=
\Proj\bigl(\mL\bvQ^{(k)}\bigr),
\)
\(
k\ge0,
\)
be the PPM iterates, and define
\(
\delta_k
:=
\delta_{\mathrm{glob}}(\bvP,\bvQ^{(k)}).
\)
Recall the contraction factor and additive error bound from
\eqref{eq:contraction_parameters}:
\[
\gamma_n:=C_1{m}/{(np\pi_0^2)}
\quad \text{ and } \quad 
\xi_n := C_2[ \exp\bigl(-c\min\{m\,np\pi_0^2,\,np\pi_0\}\bigr) + \log n/n].
\]

\begin{theorem}[Uniform-in-iteration almost-exact recovery]
\label{th:uniform_in_iteration_almost_exact}
Under the assumptions of
Theorem~\ref{th:iter_contraction_random_observation}, suppose that
\(\delta_0\le0.5-\epsilon\). Then, for all sufficiently large \(n\), on the
event of Theorem~\ref{th:iter_contraction_random_observation}, which has
probability at least \(1-O(n^{-\beta})\), every PPM iterate satisfies
\(\delta_k\le0.5-\epsilon\), the recursion
\(\delta_k\le\gamma_n\delta_{k-1}+\xi_n\) holds for every \(k\ge1\), and
\[
\sup_{k\ge1}\delta_k
\le
\max \left\{\gamma_n\delta_0+\xi_n,
\frac{\xi_n}{1-\gamma_n}\right\}
=o(1).
\]
In particular, \(\delta_{k_n}\to0\) in probability for every, possibly
data-dependent, integer-valued sequence \(k_n\ge1\).
\end{theorem}
\begin{proof}
Assumption~\ref{ass:uniform_corruption_scaling} implies
\(
\gamma_n\to0
\)
and
\(
\xi_n\to0.
\)
Indeed, \(m/(np\pi_0^2)\to0\), and since \(m\ge2\), we have
\(np\pi_0^2\to\infty\). Moreover,
\(np\pi_0\ge np\pi_0^2\). Consequently, for all sufficiently large \(n\),
\begin{equation}
\label{eq:basin_invariance_condition}
\gamma_n\left(0.5-\epsilon\right)
+
\xi_n
\le
0.5-\epsilon.
\end{equation}
On the event in
Theorem~\ref{th:iter_contraction_random_observation}, all PPM iterates
are in the positive-overlap regime and satisfy
\begin{equation}
\label{eq:delta_recursion}
\delta_k
\le
\gamma_n\delta_{k-1}
+
\xi_n,
\qquad
k\ge1.
\end{equation}
Indeed, if \(\delta_{k-1}\le0.5-\epsilon\), then the theorem and
\eqref{eq:basin_invariance_condition} imply
\(
\delta_k
\le
\gamma_n\left(0.5-\epsilon\right)
+
\xi_n
\le
0.5-\epsilon.
\)
The claim follows by induction.

Iterating~\eqref{eq:delta_recursion} gives
\(
\delta_k
\le
\gamma_n^k\delta_0
+
\frac{1-\gamma_n^k}{1-\gamma_n}\xi_n
\)
for every \(k\ge1\). Since \(\gamma_n\to0\), we have
\(0\le\gamma_n<1\) for all sufficiently large \(n\). Consequently,
\(
\delta_k
\le
\gamma_n^k\delta_0
+
\frac{\xi_n}{1-\gamma_n},
\)
where
\(
\frac{\xi_n}{1-\gamma_n}
=
\bigl(1+o(1)\bigr)\xi_n.
\)

The recursion bound is a convex combination of \(\gamma_n\delta_0+\xi_n\) and \(\xi_n/(1-\gamma_n)\): for \(k\ge1\),
\[
\gamma_n^k\delta_0+
\frac{1-\gamma_n^k}{1-\gamma_n}\xi_n
=\gamma_n^{k-1}(\gamma_n\delta_0+\xi_n)
+(1-\gamma_n^{k-1})\frac{\xi_n}{1-\gamma_n}.
\]
Consequently, on the same event of probability at least
\(1-O(n^{-\beta})\),
\[
\sup_{k\ge1}\delta_k
\le
\max\left\{\gamma_n\delta_0+\xi_n,
{\xi_n} / {(1-\gamma_n)}\right\} =o(1).
\]
In particular, \(\delta_{k_n}\to0\) in probability for every possibly
data-dependent integer-valued sequence \(k_n\ge1\).
\end{proof}
The statement of Theorem~\ref{th:uniform_in_iteration_almost_exact} is stronger than a
statement for every fixed iteration index because the same event controls the complete trajectory. It remains weaker than exact recovery:
\(n\delta_{k_n}\), the number of incorrectly recovered blocks, need not tend to zero and may still diverge.
The term \(\log n/n\) in \(\xi_n\) 
arises from the polynomial failure probability in Lemma~\ref{lem:read_two_concentration}. 
When \(c\min\{m\,np\pi_0^2,\,np\pi_0\}-\log n\to\infty\), Corollary~\ref{cor:iterative_exact_recovery} gives exact recovery. When \(\min\{m\,np\pi_0^2,\,np\pi_0\}=o(\log n)\), the term is
dominated by \(e^{-c\min\{m\,np\pi_0^2,\,np\pi_0\}}\).
\begin{corollary}[Finite-step exact recovery]
\label{cor:iterative_exact_recovery}
Under the assumptions of Theorem~\ref{th:iter_contraction_random_observation}, suppose that
\(
\delta_0\le0.5-\epsilon.
\)
For all sufficiently large \(n\), with probability at least
\[
1-O(n^{-\beta})-C_2n\exp\bigl(-c\min\{m\,np\pi_0^2,\,np\pi_0\}\bigr),
\]
the PPM iterates satisfy \(\delta_k\le\gamma_n^k\delta_0,\)
\(\gamma_n=C_1\frac{m}{np\pi_0^2},\)
\(k\ge0.\)
In particular, exact recovery occurs once
\(\gamma_n^k\delta_0<1/n.\)
For \(\delta_0>0\) and \(0<\gamma_n<1\), it suffices to take
\[
k=1 + \left\lfloor
\frac{\log(n\delta_0)}{\log(1/\gamma_n)}
\right\rfloor.\]
If \(\delta_0=0\), the error remains zero on this event. Consequently, exact recovery holds with probability tending to one whenever
\(c\min\{m\,np\pi_0^2,\,np\pi_0\}-\log n\to\infty,\)
and in particular whenever \(cnp\pi_0^2-\log n\to\infty\).
\end{corollary}
\begin{proof}
Using the notation from the proof of
Theorem~\ref{th:iter_contraction_random_observation}, the estimates there
and a union bound give
\[
\begin{aligned}
\Prob(\mathcal B_{\mathrm{deg}}\neq\varnothing)
&\le n e^{-cnp\pi_0}
\le n e^{-c\min\{m\,np\pi_0^2,\,np\pi_0\}},\\
\Prob(\mathcal B_{\mathrm{ind}}\neq\varnothing)
&\le Cn e^{-c\min\{m\,np\pi_0^2,\,np\pi_0\}}.
\end{aligned}
\]
On the complements of these events and on the two spectral-norm events used
in that proof, we have
\(
\mathcal B_{\mathrm{deg}}
=
\mathcal B_{\mathrm{ind}}
=
\varnothing.
\)
The remaining deterministic estimates give
\(
|\mathcal B_{\mathrm{sig}}|
\le
C{|T|}/{(np\pi_0^2)}
\)
and
\(
|\mathcal B_{\mathrm{dep}}|
\le
C{m}|T|/({np\pi_0^2}).
\)
Since \(m\ge2\), it follows that
\(\delta_{k+1} \le C_1{m}\delta_k/({np\pi_0^2}) = \gamma_n\delta_k\)
uniformly throughout the positive-overlap basin. Since
\(\gamma_n\to0\), this basin is invariant for all sufficiently large \(n\).
Iteration proves the claimed bound. Finally, \(n\delta_k\) is an integer, so
\(\delta_k<1/n\) implies \(\delta_k=0\).
\end{proof}

\paragraph{From almost-exact to exact recovery}
The trajectory condition \(m=o(np\pi_0^2)\) already yields
uniform-in-iteration almost-exact recovery. Finite-step exact recovery
requires the stronger event that the estimate-independent exceptional sets
\(\mathcal B_{\mathrm{deg}}\) and \(\mathcal B_{\mathrm{ind}}\) are empty.
This leads to the sufficient probability condition
\(c\min\{m\,np\pi_0^2,\,np\pi_0\}-\log n\to\infty\), where the additional \(\log n\) scale is
the cost of ruling out an error at every one of the \(n\) blocks.
By the paragraph following the proof of
Theorem~\ref{th:iter_contraction_random_observation}, the constant \(c\) in
this condition may be taken as \(c=c_0\epsilon^2\) for an absolute constant
\(c_0>0\). The branch \(np\pi_0\gtrsim\log n\) of the condition is necessary
for exact recovery by any method: if \(np\pi_0-\log n\to-\infty\), then with
high probability the graph of uncorrupted observations, an
Erd\H{o}s--R\'enyi graph \(G(n,p\pi_0)\), has an isolated vertex \(i\);
choose the least-indexed isolated vertex \(i\). Even if the indicators
and every other ground-truth block are revealed, the measurements involving
\(i\) are independent of \(\mP_i\). Under a uniform prior its conditional
recovery probability is at most \(1/m!\), so no estimator achieves exact
recovery with probability tending to one uniformly over the ground truth.

This exact-recovery condition is sufficient rather than an
information-theoretically sharp threshold. In the complete-observation,
fixed-\(m\) specialization \(p=1\), it requires \(\pi_0\) to be a
sufficiently large constant multiple of \(\sqrt{\log n/n}\), which has the
same order as the near-optimal spectral exact-recovery guarantee of
Ling~\cite{ling2022near} under uniform corruption.

The condition is, however, of the optimal order in the scale
\(np\pi_0^2/\log n\) when \(m=2\). Indeed, \(\Perm(2)\cong\mathbb Z_2\), and
for \(m=2\) Definition~\ref{def:uniform_corruption_model} coincides with the
censored block model on \(G(n,p)\) of~\cite{abbe2014decoding} with flip
probability \(\phi:=(1-\pi_0)/2\), so that \(1-2\phi=\pi_0\). For
\(p=\alpha\log n/n\) with fixed \(\alpha>0\) and fixed \(\pi_0\in(0,1)\), exact
recovery in that model is impossible when
\(\alpha(\sqrt{1-\phi}-\sqrt{\phi})^2<1\)~
\mbox{\cite[Theorem~7]{hajek2016achieving}} and is achieved by semidefinite programming when
\(\alpha(\sqrt{1-\phi}-\sqrt{\phi})^2>1\)~\mbox{\cite[Theorem~6]{hajek2016achieving}}.
Since
\(
(\sqrt{1-\phi}-\sqrt{\phi})^2
=
1-\sqrt{1-\pi_0^2}
\in[\pi_0^2/2,\pi_0^2],
\)
in this fixed-parameter logarithmic-degree regime, exact recovery requires \(np\pi_0^2\) at least of order \(\log n\). Thus,  the sufficient condition \(c\min\{m\,np\pi_0^2,\,np\pi_0\}-\log n\to\infty\), which for \(m=2\) and \(\pi_0\le1/2\) reads \(2cnp\pi_0^2-\log n\to\infty\), has the optimal order in this case, although the constant is not sharp. For growing \(m\), our results additionally require \(m=o(np\pi_0^2)\), arising from initialization and contraction. Depending on the parameter scaling, either this requirement or the all-block  probability condition can be more restrictive. We do not pursue the information-theoretic question here.

\subsection{Spectral Initialization}

\label{sec:spectral_initialization}

The reference-block spectral initializer has aligned block error \newline
\(O_{\mathbb P}(m/(np\pi_0^2))\) under \(np\ge C_0\log n\) and \(m=o(np\pi_0^2)\). These are the same scaling assumptions as in  the contraction theorem. The degree condition enters only through Lemma~\ref{lem:sparse_noise_spectral_norm} and matches the Erd\H{o}s--R\'enyi connectivity scale. This logarithmic scale is necessary for exact recovery of every block, but our proof does not show that it is necessary for basin entry.

The argument is motivated by the reference-block spectral initialization analysis in~\cite{chen2016projected}, but is formulated directly for the uniform corruption model and the globally aligned block error used here. Under additive Gaussian noise, Nguyen and Zhang~\cite{nguyen2025novel} prove a  minimax-optimal error bound using an anchor that aggregates all eigenspace blocks.  We analyze the reference-block initializer used by PPM and show that it enters the positive-overlap basin.

Let
\(
\mlamb\in\mathbb R^{nm\times m}
\)
have as its columns an orthonormal basis for the leading \(m\)-dimensional
eigenspace of \(\mL\), and write
\(
\mlamb
=
[\mlamb_1^\top, 
\ldots, 
\mlamb_n^\top
]^{\top},
\)
\(
\mlamb_i\in\mathbb R^{m\times m}.
\)
We define the spectral initializer by
\begin{equation}
\label{eq:spectral_initializer}
\mQ_i^{(0)}
:=
\Proj\bigl(n\mlamb_i\mlamb_1^\top\bigr),
\qquad
i\in[n].
\end{equation}
This definition is independent of the choice of orthonormal basis for the
leading eigenspace.

\begin{theorem}[Vanishing block error for the spectral initializer]
\label{th:spectral_initialization}
Suppose that \(\mL\) follows
Definition~\ref{def:uniform_corruption_model}, with
\(m=m_n\), \(p=p_n\), and \(\pi_0=\pi_{0,n}\). Assume that
\(np\ge C_0\log n\) for a sufficiently large absolute constant \(C_0>0\),
and define
\(
r_n:={m}/{(np\pi_0^2)}.
\)
Suppose that \(r_n\to0\). Let \(\bvQ^{(0)}\) be the spectral initializer
defined in~\eqref{eq:spectral_initializer}. Then, for every fixed
\(\beta>0\), there exists a constant \(C_\beta>0\), independent of \(n\),
such that, for all sufficiently large \(n\) and every \(\tau\in(0,1)\),
\[
\Prob\left(
\delta_{\mathrm{glob}}\bigl(\bvP,\bvQ^{(0)}\bigr)>\tau
\right)
\le
O(n^{-\beta})+C_\beta\frac{r_n}{\tau}.
\]
Consequently,
\(
\delta_{\mathrm{glob}}\bigl(\bvP,\bvQ^{(0)}\bigr)
=
O_{\mathbb P}\bigl(m/(np\pi_0^2)\bigr).
\)
Moreover,
\[
\EE\!\left[\delta_{\mathrm{glob}}\bigl(\bvP,\bvQ^{(0)}\bigr)\right] =O(r_n).
\]
The aligned block error therefore converges to zero in probability, and for every fixed \(\epsilon\in(0,0.5)\),
\[
\Prob\left(
\delta_{\mathrm{glob}}\bigl(\bvP,\bvQ^{(0)}\bigr)
\le0.5-\epsilon
\right)
\ge
1-O(n^{-\beta})-C_{\beta,\epsilon}r_n,
\]
and this probability tends to one.
\end{theorem}

\begin{proof}
Let
\(
\boldsymbol{D}
:=
\operatorname{diag}(\mP_1,\dots,\mP_n)
\)
and consider the gauged measurement matrix
\(
\mL^\sharp
:=
\boldsymbol{D}^\top\mL\boldsymbol{D}.
\)
For \(i<j\), define
\(
\boldsymbol{V}_{ij}
:=
\mP_i^\top\mU_{ij}\mP_j.
\)
Under the uniform corruption model, invariance under left and right
multiplication by fixed permutation matrices implies that the matrices
\(\boldsymbol{V}_{ij}\) remain independent and uniformly distributed on
\(\Perm(m)\). In particular,
\(
\EE[\boldsymbol{V}_{ij}]
=
m^{-1}\vone\vone^\top.
\)
Thus,
\(
\mL^\sharp_{ij}
=
A_{ij}
\left(
C_{ij}\mI_m
+
(1-C_{ij})\boldsymbol{V}_{ij}
\right).
\)

Let
\(
\boldsymbol{J}
:=
\frac{1}{m}\vone\vone^\top
\)
and
\(
\boldsymbol{K}
:=
\pi_0\mI_m+(1-\pi_0)\boldsymbol{J}.
\)
Since
\(
\EE[\boldsymbol{V}_{ij}]
=
\boldsymbol{J},
\)
we have
\(
\EE[\mL^\sharp]
=
p\bigl(\vone_n\vone_n^\top-\mI_n\bigr)
\otimes\boldsymbol{K}.
\)

The matrix \(\boldsymbol{K}\) has eigenvalue \(1\) in the direction
\(\vone\) and eigenvalue \(\pi_0\) on its orthogonal complement. The matrix
\(\vone_n\vone_n^\top-\mI_n\) has eigenvalue \(n-1\) in the direction
\(\vone_n\) and eigenvalue \(-1\) on its orthogonal complement. It follows
that the leading \(m\)-dimensional eigenspace of
\(\EE[\mL^\sharp]\) is
\(
\operatorname{span}(\vone_n)\otimes\mathbb R^m.
\)
The \(m\)-th largest eigenvalue, $\lambda_m$, of
\(\EE[\mL^\sharp]\) is \(p(n-1)\pi_0\), whereas its
\((m+1)\)-st largest eigenvalue, $\lambda_{m+1}$, is \(-p\pi_0\). Hence, the spectral gap is
\begin{equation}
\label{eq:population_spectral_gap}
\lambda_m - \lambda_{m+1} = np\pi_0.
\end{equation}
Consequently, the leading rank-\(m\) spectral projector of
\(\EE[\mL]\) is
\(
\boldsymbol{\Pi}_\star
=
\boldsymbol{D}
\left(
\vone_n\vone_n^\top\otimes\mI_m
\right)
\boldsymbol{D}^\top/n
=
\bvP\bvP^\top/n.
\)
We next control the empirical leading spectral projector. In the gauged
coordinates, the off-diagonal blocks have the form
\(
A_{ij}\boldsymbol{Y}_{ij},
\)
\(
\boldsymbol{Y}_{ij}
:=
C_{ij}\mI_m+(1-C_{ij})\boldsymbol{V}_{ij},
\)
where
\(
\|\boldsymbol{Y}_{ij}\|_2=1,
\)
\(
\EE[\boldsymbol{Y}_{ij}]
=
\boldsymbol{K}.
\)
Since
\(
m/n=p\pi_0^2r_n\le r_n\to0,
\)
we have \(m\le n\) for all sufficiently large \(n\). Therefore,
Lemma~\ref{lem:sparse_noise_spectral_norm} applies to \(\mL^\sharp\). 
Since conjugation by \(\boldsymbol{D}\) preserves the
spectral norm, let \(\mathcal E\) denote the event on which
\begin{equation}
\label{eq:spectral_initialization_noise}
\|\mL-\EE[\mL]\|_2
\le
C_\beta\sqrt{np}.
\end{equation}
Then
\(
\Prob(\mathcal E^c)=O(n^{-\beta}).
\)
Furthermore,
\(
\pi_0\sqrt{np}
=
\sqrt{{m}/{r_n}}
\to\infty,
\)
because \(m\ge2\) and \(r_n\to0\).
Therefore,
\(
{\|\mL-\EE[\mL]\|_2}/{(np\pi_0)}
\to0
\)
on the preceding event. By Weyl's inequality and
\eqref{eq:population_spectral_gap}, the empirical \(m\)-th and
\((m+1)\)-st eigenvalues are consequently separated for all sufficiently
large \(n\). Thus the leading rank-\(m\) empirical spectral projector is
unique on this event.

Let
\(
\widehat{\boldsymbol{\Pi}}
:=
\mlamb\mlamb^\top
\)
be the empirical leading rank-\(m\) spectral projector and set
\(
\boldsymbol{\Delta}
:=
\widehat{\boldsymbol{\Pi}}-\boldsymbol{\Pi}_\star.
\)
By the Davis--Kahan theorem~\cite[Theorem~2]{yu2015useful} and
\eqref{eq:population_spectral_gap}--\eqref{eq:spectral_initialization_noise},
\[
\|\boldsymbol{\Delta}\|_F^2
\le
C_\beta m
\frac{\|\mL-\EE[\mL]\|_2^2}{n^2p^2\pi_0^2}
\le
C_\beta\frac{m}{np\pi_0^2}
=
C_\beta r_n
\ \text{ on the event } \mathcal E.
\]
For any permutation matrix \(\boldsymbol{T}\), rounding satisfies 
\begin{equation}
\label{eq:spectral_rounding_separation}
\Proj(\boldsymbol{M})\neq\boldsymbol{T}
\quad\Longrightarrow\quad
\|\boldsymbol{M}-\boldsymbol{T}\|_F\ge1.
\end{equation}
Indeed, let
\(\boldsymbol{Q}:=\Proj(\boldsymbol{M})\neq\boldsymbol{T}\) and set
\(
d
:=
m-\langle\boldsymbol{T},\boldsymbol{Q}\rangle.
\)
Two distinct permutation matrices differ in at least two positions, so
\(d\ge2\), and
\(
\|\boldsymbol{Q}-\boldsymbol{T}\|_F^2=2d.
\)
Writing
\(\boldsymbol{M}=\boldsymbol{T}+\boldsymbol{E}\), optimality of
\(\boldsymbol{Q}\) gives
\(
\langle\boldsymbol{E},
\boldsymbol{Q}-\boldsymbol{T}\rangle
\ge d.
\)
The Cauchy--Schwarz inequality therefore yields
\(
\|\boldsymbol{E}\|_F
\ge
\sqrt{{d}/{2}}
\ge1,
\)
proving~\eqref{eq:spectral_rounding_separation}. 
The population and empirical reference blocks satisfy
\[
n(\boldsymbol{\Pi}_\star)_{i1}=\mP_i\mP_1^\top,
\qquad
n\widehat{\boldsymbol{\Pi}}_{i1}=n\mlamb_i\mlamb_1^\top.
\]
By~\eqref{eq:spectral_rounding_separation},
\(
\boldsymbol{1}\{\mQ_i^{(0)}\ne\mP_i\mP_1^\top\}
\le\|n\boldsymbol{\Delta}_{i1}\|_F^2.
\)

Since right alignment by \(\mP_1\) is admissible in the definition of
\(\delta_{\mathrm{glob}}\),
\begin{equation}
\delta_{\mathrm{glob}}\bigl(\bvP,\bvQ^{(0)}\bigr)
\le
\frac{1}{n}
\sum_{i=1}^n
\boldsymbol{1}
\left\{
\mQ_i^{(0)}\neq\mP_i\mP_1^\top
\right\}
\le
n\sum_{i=1}^n
\|\boldsymbol{\Delta}_{i1}\|_F^2.
\label{eq:spectral_initialization_block_error}
\end{equation}

For \(r\in[n]\), define
\(
s_r
:=
n\sum_{i=1}^n
\|\boldsymbol{\Delta}_{ir}\|_F^2.
\)
Let
\(
\boldsymbol{\Delta}^{\sharp}
:=
\boldsymbol{D}^{\top}\boldsymbol{\Delta}\boldsymbol{D}.
\)
Since
\(
\boldsymbol{\Delta}^{\sharp}_{ir}
=
\mP_i^\top\boldsymbol{\Delta}_{ir}\mP_r,
\)
we have
\(
\|\boldsymbol{\Delta}^{\sharp}_{ir}\|_F
=
\|\boldsymbol{\Delta}_{ir}\|_F.
\)
The distribution of the gauged matrix \(\mL^\sharp\), as well as the
event \(\mathcal E\), is invariant under simultaneous permutations of the
\(n\) block indices. Since the leading spectral projector is uniquely
defined on \(\mathcal E\), it follows that
\(
s_1\boldsymbol{1}_{\mathcal E},
\dots,
s_n\boldsymbol{1}_{\mathcal E}
\)
are identically distributed. 
Furthermore,
\(
\frac{1}{n}\sum_{r=1}^n s_r
=
\|\boldsymbol{\Delta}\|_F^2.
\)
Therefore,
\(
\EE\left[
s_1\boldsymbol{1}_{\mathcal E}
\right]
=
\EE\left[
\frac{1}{n}
\sum_{r=1}^n
s_r\boldsymbol{1}_{\mathcal E}
\right]
\le
C_\beta r_n.
\)
For every \(\tau>0\), Markov's inequality gives
\(
\Prob\left(
\mathcal E
\cap
\{s_1>\tau\}
\right)
\le
C_\beta\frac{r_n}{\tau}.
\)
Combining this estimate with
\eqref{eq:spectral_initialization_block_error} and
\(\Prob(\mathcal E^c)=O(n^{-\beta})\) gives
\(
\Prob\left(
\delta_{\mathrm{glob}}\bigl(\bvP,\bvQ^{(0)}\bigr)>\tau
\right)
\le
O(n^{-\beta})
+
C_\beta {r_n}/{\tau}.
\)
Since the block error is at most one, the same estimates also give
\[
\EE\!\left[\delta_{\mathrm{glob}}\bigl(\bvP,\bvQ^{(0)}\bigr)\right]
\le C_\beta r_n+O(n^{-\beta}).
\]
Here \(r_n\ge m/n\ge2/n\), so taking any fixed \(\beta>1\) gives
\(\EE[\delta_{\mathrm{glob}}(\bvP,\bvQ^{(0)})]=O(r_n)\).
The probability bound gives \(O_{\mathbb P}(r_n)\) by taking
\(\tau=Mr_n\) and then \(M\to\infty\). 
Taking \(\tau=0.5-\epsilon\) gives the positive-overlap conclusion.

For the random-reference extension in Remark~\ref{rem:beyond_uniform_law}, define \(\mQ_i^{(0)} = \Proj(n\mlamb_i\mlamb_\ell^\top)\), where \(\ell\) is uniform on \([n]\) and independent of the measurements.
No exchangeability is needed: on \(\mathcal E\), conditional on the
measurements, \(\EE_\ell[s_\ell]=\|\boldsymbol{\Delta}\|_F^2\).
The same expectation and Markov bounds therefore apply, and the uniform
contraction theorem gives the corresponding end-to-end guarantees.
\end{proof}

\begin{corollary}[End-to-end almost-exact and finite-step exact recovery]
\label{cor:end_to_end_spectral_ppm}
Suppose that Assumption~\ref{ass:uniform_corruption_scaling} holds and that \(np\ge C_0\log n\). Let \(\bvQ^{(0)}\) be the spectral initializer and \((\bvQ^{(k)})_{k\ge0}\) its PPM trajectory. Set \(\delta_k:=\delta_{\mathrm{glob}}(\bvP,\bvQ^{(k)})\). Fix \(\epsilon\in(0,0.5)\) and \(\beta>0\).

Then there exist constants \(c,C_1,C_2,C_{\beta,\epsilon}>0\), independent of \(n\), such that, upon setting
 \[
r_n:=\frac{m}{np\pi_0^2},
\qquad
\gamma_n:=C_1r_n,
\qquad
\xi_n:=C_2 \left[
e^{-c\min\{m\,np\pi_0^2,\,np\pi_0\}}
+
\frac{\log n}{n}
\right],
\]
the following conclusions hold for all sufficiently large \(n\).
We have \(r_n,\gamma_n,\xi_n\to0\), and, with probability at least
\(1-O(n^{-\beta})-C_{\beta,\epsilon}r_n\), the initializer satisfies
\(\delta_0\le0.5-\epsilon\) and, simultaneously for every \(k\ge1\),
\(
\delta_k
\le
\gamma_n^k\delta_0
+
{(1-\gamma_n^k)}\xi_n/{(1-\gamma_n)}.
\)
In particular, with the same probability,
\(\sup_{k\ge1}\delta_k \le \max\{\gamma_n\delta_0+\xi_n,\xi_n/(1-\gamma_n)\} =o(1),\)
so spectrally initialized PPM achieves uniform-in-iteration almost-exact recovery. 
Since \(\delta_0=O_{\mathbb P}(r_n)\), the same bound yields
\[
\sup_{k\ge1}\delta_k=O_{\mathbb P}(r_n^2+\xi_n),
\qquad
\delta_k=O_{\mathbb P}(r_n^{k+1}+\xi_n)
\quad\text{for every fixed }k\ge0.
\]
Thus, each fixed PPM refinement adds a factor \(r_n\) to the initialization error term until the additive floor is reached. 
Moreover, with probability at least
\(1-O(n^{-\beta})-C_{\beta,\epsilon}r_n - \)
\(C_2n e^{-c\min\{m\,np\pi_0^2,\,np\pi_0\}},\)
we have
\(\delta_k\le\gamma_n^k\delta_0\)
simultaneously for every \(k\ge0\), and exact recovery occurs once
\(\gamma_n^k\delta_0<1/n.\)
Consequently, whenever
\(c\min\{m\,np\pi_0^2,\,np\pi_0\}-\log n\to\infty\) this probability tends to one.
\end{corollary}

This follows by combining
Theorem~\ref{th:spectral_initialization},
Theorem~\ref{th:iter_contraction_random_observation},
Theorem~\ref{th:uniform_in_iteration_almost_exact}, and
Corollary~\ref{cor:iterative_exact_recovery}.

\appendix

\section{Partial Permutations: Formulation and Scope of the Theory}
\label{sec:partial_permutations}

Partial permutation synchronization arises when different observations contain possibly different subsets of a common latent universe. Correct pairwise correspondences therefore match only latent elements visible in both observations. After embedding these correspondences as partial bijections of the common
universe, the feasible transformations can be viewed as elements of the
symmetric inverse semigroup~\cite{arrigoni2017synchronization}. Algorithmic approaches include spectral methods~\cite{maset2017practical,
arrigoni2017synchronization}, alternating-minimization
methods~\cite{zhou2015multi}, convex relaxation~\cite{chen2014near},
nonnegative factorization~\cite{bernard2018synchronisation},
higher-order projected-power methods~\cite{bernard2018higher}, and a scalable
combination of cycle-edge message passing with weighted projected-power
iterations~\cite{li2022fast}. Related methods filter corrupted keypoint
matches using cluster-consistency statistics~\cite{shi2021scalable}. When the visible supports vary, the feasible
transformations do not form a group, so the full-permutation theory does not
extend automatically.

For integers \(1\le r\le m\), define the set of row-injection matrices
\[
\operatorname{Inj}(r,m)
:=
\left\{
\mQ\in\{0,1\}^{r\times m}:
\mQ\boldsymbol{1}_m=\boldsymbol{1}_r,\quad
\mQ^\top\boldsymbol{1}_r\le\boldsymbol{1}_m
\right\},
\]
where the final inequality is interpreted componentwise. Thus, every row of
\(\mQ\) contains exactly one nonzero entry and every column contains at most
one nonzero entry.

Suppose that observation \(i\) contains \(m_i\) elements from a common latent
universe of size \(m\). Its correspondence with the latent universe is
represented by
\(
\mP_i\in\operatorname{Inj}(m_i,m).
\)
Define its latent visibility set and visibility mask by
\(
S_i
:=
\left\{
k\in[m]:(\mP_i^\top\mP_i)_{kk}=1
\right\},
\)
\(
\mD_i
:=
\mP_i^\top\mP_i.
\)
The correct relative correspondence between observations \(i\) and \(j\) is
\(
\mP_i\mP_j^\top\in\{0,1\}^{m_i\times m_j}.
\)
This matrix has at most one nonzero entry in each row and column, but it need
not have one nonzero entry in every row: a row is zero when the corresponding
element of \(S_i\) is not visible in observation \(j\).

Let
\(
M:=\sum_{i=1}^n m_i,
\)
\(
\bvP
:=
[\mP_1^\top, \ldots,  \mP_n^\top]^{\top}
\in\mathbb R^{M\times m}.
\)
With identity diagonal blocks, the noiseless block matrix is
\(\mL_{\mathrm{clean}}^{\mathrm{full}}=\bvP\bvP^\top\). Under the
zero-diagonal convention used by PPM, we instead write
\(
\mL_{\mathrm{clean}}
:=
\bvP\bvP^\top
-
\operatorname{blkdiag}
\bigl(\mI_{m_1},\ldots,\mI_{m_n}\bigr).
\)

\paragraph{Optimization and blockwise projection}
For a fully specified measurement matrix, a least-squares formulation is
\[
\min_{\substack{\bvQ=(\mQ_1^\top,\ldots,\mQ_n^\top)^\top\\
                 \mQ_i\in\operatorname{Inj}(m_i,m)}}
\|\mL-\bvQ\bvQ^\top\|_F^2.
\]
If some pairs are unobserved, their blocks should be omitted from this
least-squares loss rather than treated as measured zeros. The trace
iteration below still uses zero-filled unobserved blocks. Unlike the full-permutation case, even the fully specified problem is not generally equivalent to maximizing \(\tr(\bvQ^\top\mL\bvQ)\).
Indeed,
\(\|\mL-\bvQ\bvQ^\top\|_F^2 = \|\mL\|_F^2 - 2\tr(\bvQ^\top\mL\bvQ) + \|\bvQ^\top\bvQ\|_F^2.\)
Define
\(d_k(\bvQ) := \sum_{i=1}^n(\mQ_i^\top\mQ_i)_{kk}.\)
Then
\(
\bvQ^\top\bvQ
=
\operatorname{diag}\bigl(d_1(\bvQ),\ldots,d_m(\bvQ)\bigr)
\)
and hence
\(
\|\bvQ^\top\bvQ\|_F^2
=
\sum_{k=1}^m d_k(\bvQ)^2,
\)
which can vary with \(\bvQ\).

For the trace objective, define the blockwise projected iteration by 
\[
\mQ_i^{(t+1)}\in\operatorname*{argmax}_{\mQ\in\operatorname{Inj}(m_i,m)}
\tr\bigl(\mQ^\top(\mL\bvQ^{(t)})_i\bigr).
\]
This is an ordinary assignment problem. Indeed, for
\(\mA\in\mathbb R^{r\times m}\), set
\(\widehat{\mA}:=\bigl[\mA^\top,\boldsymbol{0}_{m\times(m-r)}\bigr]^\top\)
and choose
\(\widehat{\mQ}\in\operatorname*{argmax}_{\mR\in\Perm(m)}
\tr(\mR^\top\widehat{\mA})\).
The matrix formed by the first \(r\) rows of \(\widehat{\mQ}\) belongs to
\(\operatorname*{argmax}_{\mQ\in\operatorname{Inj}(r,m)}
\tr(\mQ^\top\mA)\), because every row-injection matrix can be completed to
a full permutation matrix and the added zero rows do not affect the
objective value.

For two partial-permutation block vectors
\(\bvP=(\mP_1^\top,\ldots,\mP_n^\top)^\top\) and
\(\bvQ=(\mQ_1^\top,\ldots,\mQ_n^\top)^\top\), define 
\(
\delta_{\mathrm{part}}(\bvP,\bvQ)
:=
\min_{\mH\in\Perm(m)}
\frac{1}{n}
\left|
\left\{
i\in[n]:\mP_i\neq\mQ_i\mH
\right\}
\right|.
\)
This is the natural analogue of
\(\delta_{\mathrm{glob}}\) because the relative correspondences are
unchanged under a common right permutation.

\begin{proposition}[Reduction under a common visible support]
\label{prop:common_support_partial_permutations}
Suppose that there is a fixed set \(S\subseteq[m]\), with \(|S|=r\), such
that \(S_i=S\) for every \(i\in[n]\). Then the observed synchronization
problem reduces exactly to permutation synchronization of size \(r\).

For \(r=1\), the reduced permutation is unique, so recovery modulo
a common right permutation is immediate. For \(r\ge2\), if the reduced
relative observations follow the uniform corruption model of Definition~\ref{def:uniform_corruption_model}, with
permutation size \(r\), then the full recovery theory applies to the reduced
PPM. Equivalently, the same conclusions hold for the rectangular iteration
when every iterate is constrained to have one fixed common support. They are
not asserted for an unrestricted projection over
\(\operatorname{Inj}(r,m)\), which may change the support.
\end{proposition}

\begin{proof}
Choose
\(\mE_S\in\operatorname{Inj}(r,m)\)
whose rows enumerate \(S\). For every \(i\), there is a unique
\(\mR_i\in\Perm(r)\) such that \(\mP_i=\mR_i\mE_S\), and define
\(
\bvR:=[\mR_1^\top,\ldots,\mR_n^\top]^\top.
\)
Since \(\mE_S\mE_S^\top=\mI_r\),
\(
\mP_i\mP_j^\top
=
\mR_i\mE_S\mE_S^\top\mR_j^\top
=
\mR_i\mR_j^\top.
\)
Thus, the relative observations are exactly those generated by the reduced
permutations \(\mR_1,\ldots,\mR_n\).

To verify the iteration, suppose that
\(\mQ_j^{(t)}=\widehat{\mR}_j^{(t)}\mE_S\), where
\(\widehat{\mR}_j^{(t)}\in\Perm(r)\), and define
\(
\widehat{\bvR}^{(t)}
:=
[(\widehat{\mR}_1^{(t)})^\top,\ldots,
  (\widehat{\mR}_n^{(t)})^\top]^\top.
\)
Then
\(
(\mL\bvQ^{(t)})_i
=
(\mL\widehat{\bvR}^{(t)})_i\mE_S.
\)
For every \(\mR\in\Perm(r)\) and
\(\mA\in\mathbb R^{r\times r}\),
\(
\langle\mR\mE_S,\mA\mE_S\rangle
=
\langle\mR,\mA\rangle.
\)
Consequently, projection onto
\(\{\mR\mE_S:\mR\in\Perm(r)\}\)
is exactly the size-\(r\) permutation projection. The one-step,
uniform-contraction, finite-step exact-recovery, and spectral-initialization
results therefore follow with \(m\) replaced by \(r\).
\end{proof}

\begin{corollary}[End-to-end almost-exact and finite-step exact recovery under common support]
\label{cor:common_support_partial_recovery}
Suppose that Proposition~\ref{prop:common_support_partial_permutations} holds
with support size \(r=r(n)\ge2\), and that the reduced \(r\times r\) observations
follow the uniform corruption model. Assume
Assumption~\ref{ass:uniform_corruption_scaling} with \(m\) replaced by \(r\),
and suppose that \(np\ge C_0\log n\).

Fix \(\mE_S\in\operatorname{Inj}(r,m)\) whose rows enumerate \(S\), and 
define \(\mR_i\in\Perm(r)\) by
\(\mP_i=\mR_i\mE_S\). Set
\(
\bvR:=[\mR_1^\top,\ldots,\mR_n^\top]^\top.
\)
Let
\(
\widehat{\bvR}^{(0)}
=
[(\widehat{\mR}_1^{(0)})^\top,\ldots,
  (\widehat{\mR}_n^{(0)})^\top]^\top
\)
be the reduced spectral initializer, 
\(
\widehat{\bvR}^{(k+1)}
=
\Proj(\mL\widehat{\bvR}^{(k)}),
\)
and define 
\(
\mQ_i^{(k)}=\widehat{\mR}_i^{(k)}\mE_S
\)
and
\(
\bvQ^{(k)}
=
[(\mQ_1^{(k)})^\top,\ldots,(\mQ_n^{(k)})^\top]^\top.
\)
Fix
\(\epsilon\in(0,0.5)\) and \(\beta>0\).

There exist positive constants \(c,C_1,C_2,C_{\beta,\epsilon}\), independent of \(n\), such that, upon setting

\(
\gamma_n:=C_1{r}/{(np\pi_0^2)},
\)
\(\xi_n:= C_2[e^{-c\min\{r\,np\pi_0^2,\,np\pi_0\}} + {\log n}/{n}],\)
the following holds for all sufficiently large \(n\). Assumption
\ref{ass:uniform_corruption_scaling}, with \(m\) replaced by \(r\), implies
\(\gamma_n,\xi_n\to0\).
Then, with probability at least
\(1-O(n^{-\beta})-C_{\beta,\epsilon}{r}/{(np\pi_0^2)}\),
the initializer satisfies
\(
\delta_{\mathrm{part}}(\bvP,\bvQ^{(0)})
\le
0.5-\epsilon,
\)
and, simultaneously for every \(k\ge1\),
\(
\delta_{\mathrm{part}}(\bvP,\bvQ^{(k)})
\le
\gamma_n^k\delta_{\mathrm{part}}(\bvP,\bvQ^{(0)})
+
{(1-\gamma_n^k)}\xi_n/{(1-\gamma_n)}.
\)
With the same probability,
\(
\sup_{k\ge1}
\delta_{\mathrm{part}}(\bvP,\bvQ^{(k)})
\le
\gamma_n\delta_{\mathrm{part}}(\bvP,\bvQ^{(0)})
+
{\xi_n}/{(1-\gamma_n)}
=
o(1),
\)
so the lifted PPM trajectory achieves uniform-in-iteration almost-exact
recovery under common support.
Moreover, with probability at least
\(1 - O(n^{-\beta}) - C_{\beta,\epsilon}{r}/{(np\pi_0^2)} - \)
\(C_2n e^{-c\min\{r\,np\pi_0^2,\,np\pi_0\}},\)
simultaneously for every \(k\ge0\),
\(
\delta_{\mathrm{part}}(\bvP,\bvQ^{(k)})
\le
\gamma_n^k
\delta_{\mathrm{part}}(\bvP,\bvQ^{(0)}).
\)
On this event, exact recovery occurs once
\(
\gamma_n^k
\delta_{\mathrm{part}}(\bvP,\bvQ^{(0)})<1/n.
\)
Consequently, the preceding probability
tends to one whenever
\(
c\min\{r\,np\pi_0^2,\,np\pi_0\}-\log n\to\infty.
\)
\end{corollary}

\begin{proof}
For every \(\mK\in\Perm(r)\), there exists
\(\mH\in\Perm(m)\) such that
\(
\mE_S\mH=\mK\mE_S.
\)
Whenever this identity holds, for every \(j\in[n]\),
\(
\mP_j=\mQ_j^{(k)}\mH
\leftrightarrow
\mR_j=\widehat{\mR}_j^{(k)}\mK.
\)
Consequently,
\(
\delta_{\mathrm{part}}(\bvP,\bvQ^{(k)})
\le
\delta_{\mathrm{glob}}(\bvR,\widehat{\bvR}^{(k)}).
\)
Conversely, a reduced alignment can be chosen to match at least one
block, and hence
\(
\delta_{\mathrm{glob}}(\bvR,\widehat{\bvR}^{(k)})\le1-1/n.
\)
Therefore, a minimizing ambient permutation in the definition of
\(\delta_{\mathrm{part}}\) must match at least one block. If
\(
\mP_i=\mQ_i^{(k)}\mH
\)
for some \(i\), then \(\mH\) preserves \(S\) setwise, so there is a
\(\mK\in\Perm(r)\) such that
\(
\mE_S\mH=\mK\mE_S.
\)
It follows that
\(
\delta_{\mathrm{glob}}(\bvR,\widehat{\bvR}^{(k)})
\le
\delta_{\mathrm{part}}(\bvP,\bvQ^{(k)}).
\)
Thus,
\(
\delta_{\mathrm{part}}(\bvP,\bvQ^{(k)})
=
\delta_{\mathrm{glob}}(\bvR,\widehat{\bvR}^{(k)}).
\)
The claimed conclusions follow from
Theorems~\ref{th:iter_contraction_random_observation}, \ref{th:uniform_in_iteration_almost_exact}, and~\ref{th:spectral_initialization} and Corollary~\ref{cor:iterative_exact_recovery}, with \(m\) replaced by \(r\).
Since the error is defined up to an ambient right permutation, the estimator
may lift the reduced iterates to any fixed \(r\)-element support; it needn't
know the true set \(S\).
\end{proof}

\paragraph{Varying supports}
For varying supports, \(\mP_j^\top\mP_j=\mD_j\) is a diagonal visibility mask rather than the identity, whereas \(\mP_j\mP_j^\top=\mI_{m_j}\) still holds. Consequently, even at the ground truth the signal is coordinate-dependent: the scalar signal count in the full-permutation analysis is replaced by a coordinatewise co-visibility count. Extending Lemma~\ref{lem:signal_gap_random_observation}  to this setting requires a signed gap for each visible coordinate between correctly and incorrectly
estimated neighbors, together with uniform control of the remaining
perturbation terms.
Let  \(G_{ij}\in\{0,1\}\) indicate that 
\((i,j)\) is observed and uncorrupted, with \(G_{ii}=0\), and define
\(
\mH_i:=\sum_{j\ne i}G_{ij}\mD_j,
\)
\(
h_{ik}:=(\mH_i)_{kk},
\)
\(
\kappa_i:=\min_{k\in S_i}h_{ik}.
\)
At the ground truth, where \(\mQ_j=\mP_j\), the uncorrupted component of the block-\(i\) score is
\(
\sum_{j\ne i}G_{ij}\mP_i\mP_j^\top\mP_j
=
\mP_i\mH_i.
\)
Thus, \(h_{ik}\) is the number of observed uncorrupted neighbors of \(i\)
that also contain latent element \(k\).

\begin{lemma}[Deterministic co-visibility margin]
\label{lem:partial_cov_visibility_margin}
Fix \(i\), and suppose that a block score has the form
\(\mA_i=\mP_i\mH_i+\mE_i\). For
\(\mQ\in\operatorname{Inj}(m_i,m)\), define
\(
d_i(\mQ)
:=
\left|
\left\{
a\in[m_i]:(\mQ)_{a,:}\neq(\mP_i)_{a,:}
\right\}
\right|.
\)
Then
\(
\left\langle
\mA_i,\mP_i-\mQ
\right\rangle
\ge
\kappa_i d_i(\mQ)
-
\|\mE_i\|_F\sqrt{2d_i(\mQ)}.
\)
Consequently, if \(\kappa_i>0\) and
\(\|\mE_i\|_F<\kappa_i/\sqrt{2}\), then
\(\mP_i\) is the unique projection of \(\mA_i\) onto
\(\operatorname{Inj}(m_i,m)\).
If \(m_i=m\ge2\), let \(h_{i,(1)}\le h_{i,(2)}\) be the two smallest
co-visibility counts. In this square case the sufficient condition improves
to
\[
\|\mE_i\|_F<\frac{h_{i,(1)}+h_{i,(2)}}{2}.
\]
For arbitrary perturbations measured in Frobenius norm, the radius
\(\kappa_i/\sqrt2\) is optimal when \(m_i<m\), and the displayed radius
is optimal when \(m_i=m\ge2\).
\end{lemma}

\begin{proof}
In row \(a\), the only possibly nonzero entry of \(\mP_i\mH_i\) 
is in
the true column \(k(a)\), where its value is \(h_{i,k(a)}\). Therefore,
\[
\left\langle
\mP_i\mH_i,\mP_i-\mQ
\right\rangle
=
\sum_{\substack{a\in[m_i]\\
(\mQ)_{a,:}\neq(\mP_i)_{a,:}}}
h_{i,k(a)}
\ge
\kappa_i d_i(\mQ).
\]
Moreover,
\(\|\mP_i-\mQ\|_F^2=2d_i(\mQ)\). The claimed inequality follows from
Cauchy--Schwarz. Since \(d_i(\mQ)\ge1\) for every
\(\mQ\neq\mP_i\), the stated norm condition makes the projection gap
positive.

In the square case, every competing permutation changes \(d\ge2\)
rows. Since all co-visibility counts are nonnegative, its clean score gap is
at least
\(h_{i,(1)}+(d-1)h_{i,(2)}\ge d(h_{i,(1)}+h_{i,(2)})/2\).
Dividing by \(\|\mP_i-\mQ\|_F=\sqrt{2d}\) gives the improved radius.

To check optimality, in the rectangular case change a least-supported row
of \(\mP_i\) to an unused column; in the square case exchange the true
columns of the two least-supported rows. For the resulting competitor
\(\mQ\), let
\(g:=\langle\mP_i\mH_i,\mP_i-\mQ\rangle\) and take
\[
\mE_i=\frac{g}{\|\mQ-\mP_i\|_F^2}(\mQ-\mP_i).
\]
Its norm equals the asserted radius and the truth and competitor have
equal scores. The gap bounds above ensure that neither is beaten by a
third candidate. Hence strict inequality at these radii is necessary for
a guarantee of unique projection under arbitrary Frobenius perturbations.
\end{proof}

For example, suppose that \(S_1,\ldots,S_n\) are independent and that
\(S_i\) is uniformly distributed over the \(m_i\)-element subsets of
\([m]\). Suppose also that
\(\{G_{ij}:1\le i<j\le n\}\) are mutually independent
\(\operatorname{Bernoulli}(\theta)\) variables, independent of the supports,
where \(\theta=p\pi_0\), and set \(G_{ji}:=G_{ij}\).

Conditional on \(k\in S_i\),
\(
h_{ik}
=
\sum_{j\ne i}G_{ij}\boldsymbol{1}_{\{k\in S_j\}}
\)
is a sum of independent Bernoulli variables with mean
\(
\mu_i
:=
\EE[h_{ik}\mid k\in S_i]
=
\frac{\theta}{m}\sum_{j\ne i}m_j.
\)
Let \(\mu_{\min}:=\min_{i\in[n]}\mu_i\). For every 
\(\eta\in(0,1)\), the multiplicative Chernoff bound and
\(\Prob(k\in S_i)=m_i/m\) give
\[
\Prob\left(
k\in S_i,\ 
h_{ik}<(1-\eta)\mu_i
\right)
\le
\frac{m_i}{m}
\exp\left(-\frac{\eta^2}{2}\mu_i\right).
\]
A union bound over \(i\in[n]\) and \(k\in[m]\) therefore gives
\[
\Prob\left(
\exists\,i\in[n],\ k\in S_i:
h_{ik}<(1-\eta)\mu_i
\right)
\le
\sum_{i=1}^n
m_i\exp\left(-\frac{\eta^2}{2}\mu_i\right) 
\le
M\exp\left(-\frac{\eta^2}{2}\mu_{\min}\right).
\]
Consequently, for every fixed \(\eta\in(0,1)\), with probability tending
to one,
\(
\kappa_i\ge(1-\eta)\mu_i
\)
simultaneously for every \(i\in[n]\), provided that
\(
\frac{\eta^2}{2}\mu_{\min}-\log M\to+\infty.
\)

An end-to-end trajectory guarantee for varying supports also requires uniform bounds on the masked score perturbations and an initializer for
that model. The end-to-end  guarantees above apply to common support.

\section*{Statement on the use of AI tools}

The authors drafted all original content. AI tools were used exclusively for proofreading, language polishing, and very minor refinements. The authors reviewed all suggested edits and take full responsibility for the final content.

\bibliographystyle{siamplain}
\bibliography{ref}

\end{document}